\documentclass{article}
\usepackage[papersize={8.5in,11in},margin=1in]{geometry}

\usepackage{booktabs} 
\usepackage{libertine}
\usepackage{newtxmath}
\usepackage{zi4}
\usepackage{framed}

\usepackage{amsthm}
\usepackage{amsmath}

\usepackage{amssymb}
\usepackage{mathrsfs}
\usepackage{xspace}
\usepackage{booktabs}
\usepackage{aliascnt}
\usepackage{balance}
\usepackage{graphicx}
\usepackage{subfigure}
\usepackage{enumitem}
\usepackage{multirow}
\usepackage[pagebackref=true,colorlinks]{hyperref}
\hypersetup{linkcolor=black,filecolor=black,citecolor=black,urlcolor=black}

\theoremstyle{acmplain}
\newtheorem{theorem}{Theorem}[section]

\newaliascnt{lemma}{theorem}
\newaliascnt{assumption}{theorem}
\newaliascnt{proposition}{theorem}
\newaliascnt{corollary}{theorem}
\newaliascnt{claim}{theorem}
\newaliascnt{observation}{theorem}
\newaliascnt{definition}{theorem}
\newaliascnt{fact}{theorem}
\newaliascnt{statement}{theorem}
\newaliascnt{mechanism}{theorem}
\newaliascnt{example}{theorem}
\newaliascnt{remark}{theorem}

\newtheorem{lemma}[lemma]{Lemma}
\newtheorem{Lemma}[lemma]{Lemma}
\newtheorem*{theorem*}{Theorem}
\newtheorem{Remark}[remark]{Remark}
\newtheorem{assumption}[assumption]{Assumption}

\theoremstyle{acmdefinition}
\newtheorem{definition}[definition]{Definition}

\theoremstyle{acmplain}

\aliascntresetthe{lemma}
\aliascntresetthe{assumption}
\aliascntresetthe{proposition}
\aliascntresetthe{corollary}
\aliascntresetthe{claim}
\aliascntresetthe{observation}
\aliascntresetthe{definition}
\aliascntresetthe{fact}
\aliascntresetthe{statement}
\aliascntresetthe{mechanism}
\aliascntresetthe{example}

\newcommand{\argmax}{\operatornamewithlimits{argmax}}

\newcommand\R{\mathbb{R}}

\newcommand\Rev{\textsc{Rev}}
\newcommand\Menu{\textsc{Menu}}

\renewcommand\S{\mathcal{S}}

\newcommand\Loss{\textsc{Loss}}

\newcommand{\F}{\mathcal{F}}

\newcommand{\V}{\mathcal{V}}
\newcommand{\X}{\mathcal{X}}
\newcommand{\x}{\boldsymbol{x}}

\newcommand{\s}{\mathbf{s}}

\newcommand{\Zero}{\mathbf{0}}

\renewcommand{\v}{\boldsymbol{v}}
\newcommand{\p}{\boldsymbol{p}}

\newcommand{\A}{\mathcal{A}}

\newcommand{\Opt}{\textsc{Opt}}

\makeatletter
\renewcommand{\paragraph}{%
  \@startsection{paragraph}{4}%
  {\z@}{0.5ex \@plus 0.5ex \@minus .2ex}{-1em}%
  {\normalfont\normalsize\it}%
}
\makeatother

\usepackage[ruled]{algorithm2e} 
\usepackage[numbers]{natbib}
\usepackage{hyperref}
\usepackage{xspace}

\newcommand{\menunet}{\textsc{MenuNet}\xspace}

\begin{document}
\title{Automated Mechanism Design via Neural Networks}
\author{
Weiran Shen\\
IIIS, Tsinghua University\\
\texttt{emersonswr@gmail.com}
\and
Pingzhong Tang\\
IIIS, Tsinghua University\\
\texttt{kenshinping@gmail.com}
\and
Song Zuo\\
IIIS, Tsinghua University\\
\texttt{songzuo.z@gmail.com}
}
\date{}

\maketitle

\begin{abstract}
  Using AI approaches to automatically design mechanisms has been a central research mission at the interface of AI and economics [Conitzer and Sandholm, 2002]. Previous approaches that attempt to design revenue optimal auctions for the multi-dimensional settings fall short in at least one of the three aspects: 1) representation --- search in a space that probably does not even contain the optimal mechanism; 2) exactness --- finding a mechanism that is either not truthful or far from optimal; 3) domain dependence --- need a different design for different environment settings.

To resolve the three difficulties, in this paper, we put forward --- \menunet{} --- a unified neural network based framework that automatically learns to design revenue optimal mechanisms. Our framework consists of a mechanism network that takes an input distribution for training and outputs a mechanism, as well as a buyer network that takes a mechanism as input and output an action. Such a separation in design mitigates the difficulty to impose incentive compatibility constraints on the mechanism, by making it a rational choice of the buyer. As a result, our framework easily overcomes the previously mentioned difficulty in incorporating IC constraints and always returns exactly incentive compatible mechanisms.

We then apply our framework to a number of multi-item revenue optimal design settings, for a few of which the theoretically optimal mechanisms are unknown. We then go on to theoretically prove that the mechanisms found by our framework are indeed optimal.

To the best of our knowledge, we are the first to apply neural networks to discover optimal auction mechanisms with provable optimality.

\end{abstract}


\section{Introduction}\label{sec:intro}

Designing revenue optimal mechanisms in various settings has been a central
research agenda in economics, ever since the seminal works of
\citet{vickrey1961counterspeculation} and \citet{myerson1981optimal} in single
item auctions. Lately, designing optimal mechanisms for selling multiple items
has also been established as an important research agenda at the interface of
economics and computer sciences
\cite{conitzer2002complexity,hartline2009simple,hart2017approximate,CaiDW12a,CaiDW12b,li2013revenue,yao2014n,sandholm2015automated,yao2016solutions,tang2016optimal,tang2017optimal}

Due to diversity in the researchers' backgrounds, there are a number of quite
different angles to study this problem. The standard economics theme aims to
understand the exact optimal mechanisms in various settings. To name a few,
\citet{armstrong1996multiproduct} obtains the revenue optimal mechanisms of selling two items to one
buyer, whose valuations of the two items are perfect positively correlated (a
ray through the origin). \citet{manelli2007multidimensional} obtains partial
characterization of optimal mechanisms, in the form of extremely points in the
mechanism spaces. \citet{pavlov2011optimal} derives optimal mechanisms for two
items when the buyer has symmetric uniform distributions.
\citet{daskalakis2013mechanism} characterizes sufficient and necessary
conditions for a mechanism to optimal and derive optimal mechanisms for two
items for several valuation distributions. \citet{tang2017optimal} obtain the
revenue optimal mechanisms of selling two items, of which the valuations are
perfect negatively correlated. \citet{yao2016solutions} obtains the revenue
optimal mechanisms of selling two additive items to multiple buyers, whose
valuation towards the items are binary and independent.

Another category of research rooted in the AGT community aims to resolve the
difficulties of characterizing optimal mechanisms via the lens of algorithm
design. \citet{CaiDW12a} and \citet{alaei2012bayesian} gives algorithmic characterizations of
the optimal BIC mechanisms on discrete distributions using linear programs.
\citet{hartline2009simple,yao2014n,hart2017approximate}
find approximately optimal mechanisms in various settings.
\citet{carroll2017robustness} shows that for a certain multi-dimensional
screening problem, the worst-case optimal mechanism is simply to sell each item
separately.

The third category, at the interface of AI and economics, aims to search for the
optimal mechanisms via various AI approaches. \citet{conitzer2002complexity}
model the problem of revenue and welfare maximization as an instance of
constraints satisfaction problem (CSP) through which the optimal mechanism may
be found using various search techniques, despite its general computation
complexity. \citet{sandholm2015automated} model a restricted revenue
maximization problem (within affine maximizing auctions) as a parameter search
problem in a multi-dimensional parameter space, they find several sets of
parameters that yields good empirical revenue. \citet{dutting2019optimal} aims
to learn optimal mechanisms by repeatedly sampling from the distribution. They
obtain mechanisms that are approximately optimal and have low incentive
compatibility regret on average.

One advantage of these computational approaches is that most of them are
constructive so that one can systematically and computationally generate
optimal mechanisms. However, a difficulty for most existing works in computer
science (the second and third categories) is that mechanisms obtained this way
are either not optimal in the exact sense, or not truthful in the exact sense.
As a result, a typical economist may have a hard time to appreciate this type of
results. A more desirable approach would be constructive on one hand and be able
return exact incentive compatible and (hopefully) exact optimal mechanisms on
the other hand.

\subsection{Our methodology}
In this paper, motivated by the above observation, we aim to put forward a
computational approach that can design or assist one to design exact IC and
optimal mechanisms. We train a
neural network that represents the optimal mechanism using the
valuation distributions.
However, unlike the approach in \citet{dutting2019optimal}, we introduce another
neural network that represents buyer's behavior. In particular, this network
takes a mechanism as input, and output an action. Our network structure
resembles that of the generative adversarial nets (GAN) \cite{goodfellow2014generative} but is essentially
different because we do not need to train the buyer's network. This independent
buyer network allows us to easily model the exact IC constraints (which has been
a major difficulty in previous works) in our network and any behavior model of
this form. In contrast, \citet{dutting2019optimal} first propose to hardwire the IC constraints into
the mechanism network, which requires a lot of domain knowledge and the
structure of the networks has to be domain specific. As a result their approach
can only reproduce mechanisms in the domains where the form of the optimal
mechanism is known. To circumvent this difficulty,  they further propose to add
IC as a soft constraint so that the training objective is to minimize a linear
combination of revenue loss and the degree of IC violations. However, this would
produce mechanisms that are not IC.

Another innovation of our framework, \menunet, is that we represent a mechanism as a menu (a
list of (valuation, outcome) tuples) in the single buyer case. According to the
taxation principle \cite{vohra2011mechanism}, by simply letting the buyer do the
selection, we get an IC mechanism. An additional merit of using a menu to
represent a mechanism is that it enables explicit restrictions of the {\em menu
size} of the mechanism, which measures the degree of complexity of a mechanism
\cite{hart2013menu}.

Under the guidance of the solutions from our neural networks, one may be able to
guess the structure of the optimal solution and the prove its optimality.
Although our neural network framework cannot directly help with the optimality
proof, its high accuracy (see \autoref{table:compare}) can greatly reduce the
tremendous efforts that one often needs to guess the optimal solution.

\subsection{Our results}

We then apply our learning-aided mechanism design framework to the domain where
a seller sells two items to one buyer. In particular, we investigate the
following problems.
\begin{itemize}
\item What is the revenue optimal mechanisms when the menu size is restricted to
      a constant? To the best of our knowledge, the optimal mechanism of this
      kind is previously unknown for our setting.
\item What is optimal mechanism for the case where the valuation domain is a
      triangle? The previously studied cases on this domain all focuses on
      rectangle shaped valuation domain (expect for \citet{haghpanah2014multi}).
\item What is the revenue optimal {\em deterministic} mechanism?
\item What is the revenue optimal mechanism when the buyer has combinatorial
      value?
\end{itemize}

Some of the experimental results we obtained is shown in \autoref{table:compare}
with comparison to the exact optimal mechanisms (some of them are previously
known results, while the others are our new findings).

\begin{table}[h!]
  \begin{center}
  \begin{tabular}{| l | l | l | l |}
    \hline
    Distributions & Computed Mech \Rev \footnotemark & Optimal Mech \Rev & Optimality  \\
    \hline
    $U[0, 1]^2$ & $0.5491989$ & $(12 + 2\sqrt2) / 27$ & $\geq 99.9996\%$  \\
    \hline
    $U[0, 1]\times[0, 1.5]$ & $0.6838542$ & $(15 + 2\sqrt3) / 27$ & $\geq 99.9997\%$  \\
    \hline
    $U[0, 1]\times[0, 1.9]$ & $0.7888323$ & $(17.4 + 2\sqrt{3.8}) / 27$ & $\geq 99.9988\%$  \\
    \hline
    $U[0, 1]\times[0, 2]$ & $0.8148131$ & $22 / 27$ & $\geq 99.9997\%$  \\
    \hline
    $U[0, 1]\times[0, 2.5]$ & $0.9435182$ & $1019 / 1080$ & $\geq 99.99996\%$  \\
    \hline
    $U[0, 1]^2$ menu size $\leq 3$ & $0.5462947$ & $59 / 108$ & $\geq 99.9997\%$  \\
    \hline
    $U[0, 1]^2$ menu size $\leq 2$ & $0.5443309$ & $2\sqrt{6} / 9$ & $\geq 99.9999997\%$  \\
    \hline
    $U\{v_1, v_2 \geq 0 | v_1 / 2 + v_2 \leq 1\}$ & $0.5491225$ & $(12 + 2\sqrt2) / 27$ & $\geq 99.9857\%$  \\
    \hline
  \end{tabular}
  \caption{Comparison with optimal mechanisms, where Optimality $= \Rev/
          \Opt\Rev$.}
  \label{table:compare}
  \end{center}
\end{table}

Inspired by these empirical findings, using the techniques by
\citet{daskalakis2013mechanism} and \citet{pavlov2011property}, we then prove
the exact optimal mechanisms for the first two problems. To the best of our
knowledge, this is the first time to find the exact optimal mechanisms in these
domains, so they are of independent interests to the economics society as well.

\begin{theorem*}[Restricted Menu Size]
  The optimal mechanism for an additive buyer, $ \v \sim U[0, 1]^2$, with menu
  size no more than $3$ is to either sell the first item at price $2 / 3$ or
  sell the bundle of two items at price $5 / 6$, yielding revenue $59 / 108$.

  In particular, the optimal mechanism must be asymmetric even if the distribution is symmetric!
\end{theorem*}

\begin{theorem*}[Uniform Distribution on a Triangle\footnote{In the recent versions since 2020,
  \citet{dutting2020optimal} followed our methodology to discover and prove the
  optimal mechanisms for the uniform distribution on a shifted and scaled triangle, i.e., $\{(v_1, v_2) | v_1 / c + v_2 \leq 2, v_1 \geq 0, v_2 \geq 1\}$.}]
  The optimal mechanism for an additive buyer with value uniformly distributed
  in $\{(v_1, v_2) | v_1 / c + v_2 \leq 1, v_1, v_2 \geq 0\}$ (hence a
  correlated distribution) is as follows:
  \begin{itemize}
    \item if $c \in [1, 4/3]$, two menu items: $[(0, 0), 0]$ and $[(1, 1),
          \sqrt{c / 3}]$;
    \item if $c > 4 / 3$, three menu items: $[(0, 0), 0]$, $[(1, 1), 2c / 3 +
          \sqrt{c(c - 1)} / 3]$, and $[(1 / c, 1), 2 / 3]$.
  \end{itemize}
\end{theorem*}

\footnotetext{The computed revenue is NOT directly given by the loss of our network. Instead, we ignore the buyer network and compute the expected revenue according only to the menu given by our network.}

\section{Preliminaries}\label{sec:prelim}

  In this paper, we consider the automated mechanism design problem for the
  single-buyer multi-dimensional setting. In this section, we introduce the
  basic notions for optimal multidimensional mechanism design problem.

  \paragraph{Environment}
  The seller has $m$ heterogeneous items for sale, and the buyer has different
  private values for receiving different bundles of the items. An {\em
  allocation} of the items is specified by a vector $\x \in \X \subseteq
  [0, 1]^m$, where $x_i$ is the probability of allocating the $i$-th item to the
  buyer. An allocation $\x$ is called a {\em deterministic allocation}, if $\x
  \in \{0, 1\}^m$; otherwise a {\em randomized allocation} or a {\em lottery
  allocation}.

  A possible {\em outcome} of the mechanism consists of a valid allocation
  vector $\x \in \X$ and a monetary transfer amount $p \in \R_+$, called {\em
  payment}, from the buyer to the seller.

  With the standard {\em quasi-linear utility} assumption, the valuation
  function $v : \X \mapsto \R_+$ describes the private preference of the
  buyer, i.e., an outcome $\langle \x, p \rangle$ is (weakly) preferred than
  another outcome $\langle \x', p' \rangle$, if and only if:
  \begin{align*}
    u(\x, p; v) := v(\x) - p \geq v(\x') - p' = u(\x', p'; v).
  \end{align*}
  In other words, the outcome with the highest utility is most preferred by the
  buyer.

  \paragraph{Mechanism}
  A na\"ive mechanism (without applying the revelation principle) is
  defined by a set of {\em actions} and a mapping from the set of actions to the
  set of outcomes. Note that according to the taxation principle
  \cite{vohra2011mechanism}, simply letting the buyer do the selection, we get
  an incentive compatible mechanism. Formally,

  \begin{definition}[Na\"ive Mechanism]\label{def:naive}
    A na\"ive mechanism consists of an action set $\A$ and an associated mapping
    from any action to a possible outcome, i.e., $\langle \x, p \rangle : \A
    \mapsto \X \times \R_+$.

    In particular, there exists a special action $\bot$ meaning ``exiting the
    mechanism'' such that
    \begin{align}\label{eq:exit}\tag{\textsc{Exit}}
      \x(\bot) = \Zero, p(\bot) = 0.
    \end{align}
  \end{definition}

  In such a na\"ive mechanism, a strategy of the buyer is then a mapping from
  the set of private valuation functions to the action set, i.e., $s : \V
  \mapsto \A$. Furthermore, if the buyer is {\em rational}, then her
  strategy must maximize her utility:
  \begin{align}\label{eq:rational}\tag{\textsc{Rational}}
    s(v) \in \argmax_{s' \in \S} u(\x(s'(v)), p(s'(v)); v).
  \end{align}

  The corresponding outcomes of the actions are also known as {\em menu items}.
  Throughout this paper, we use $[\x, p]$ to denote a specific menu item, e.g., the
  zero menu item $[\Zero, 0] = [(0, \ldots, 0), 0]$ is the corresponding menu item of
  the exiting action $\bot$. Note that the na\"ive mechanism with the menu
  presentation is a very general model of the mechanism design problem. In
  particular, even when the buyer is not fully rational, as long as a buyer
  behavior is available, the mechanism designer is still able to design the menus to
  maximize his objective assuming that the buyer responses according
  to the given behavior model. The robustness of na\"ive mechanisms is indeed
  critical to the flexibility and generality of our methodology.

  \paragraph{Direct Mechanism}
  With the above definition of na\"ive mechanisms, it is hard to characterize
  all the mechanisms with certain properties, because the design of the action
  set, at first glance, could be arbitrary. One critical step in the mechanism
  design theory is to applying the celebrating revelation principle
  \cite[p.224]{nisan2007algorithmic} to restrict the set of na\"ive mechanisms
  to a considerably smaller set of mechanisms --- the {\em direct mechanisms}.
  In a direct mechanism, the action set is restricted to be identical to the set
  of valuation functions and the identity mapping also is required to be an
  optimal strategy for any rational buyer. Formally,
  \begin{definition}[Direct Mechanism]\label{def:direct}
    A direct mechanism fixes the action set $\A = \V$ and remains to specify the
    mapping from $\V$ to the set of possible outcomes.

    In addition, the identity mapping must be a utility-maximizing strategy for
    any rational buyer, which can be equivalently stated as the following {\em
    incentive compatible} \eqref{eq:ic} and {\em individually rational}
    \eqref{eq:ir} constraints:
    \begin{gather}
      v \in \argmax_{v' \in \V} u(\x(v'), p(v'); v),
        \label{eq:ic}\tag{\textsc{IC}}  \\
      u(\x(v), p(v); v) \geq 0.
        \label{eq:ir}\tag{\textsc{IR}}
    \end{gather}
  \end{definition}

  In fact, the constraints \eqref{eq:ic} and \eqref{eq:ir} are deduced from the
  constraints \eqref{eq:rational} and \eqref{eq:exit}.

  \paragraph{The Designer's Goal}
  The goal of the mechanism designer is to maximize the expectation of his
  objective $r : \X \times \R_+ \mapsto \R$, where the expectation is taken
  over his prior knowledge about the buyer's private valuation function, i.e.,
  $v \sim \F$.

  We emphasize that our methodology is not restricted to any specific objective.
  However, in this paper, we would focus on the setting with the seller's
  revenue as the objective:
  \begin{align}\label{eq:obj}\tag{\textsc{Objective}}
    r(\x, p) = p.
  \end{align}
  Because revenue-optimal mechanism design in multi-dimensional environment is a
  both challenging and widely studied problem. Hence applying our method in such
  a setting allows us to verify that (i) whether it can find the optimal or
  nearly optimal solution, and (ii) whether it can provide a simpler approach to
  a hard problem.

  \paragraph{Assumptions}
  In most sections of this paper, we will make to the following two assumptions
  (\autoref{assump:additive} and \autoref{assump:independent}). As we just
  stated, we would first verify that our method can be used to recover the
  optimal solutions to some known problems and little exact optimal solution is
  actually discovered without these two assumptions.

  \begin{assumption}[Additive Valuation Functions]\label{assump:additive}
    The buyer's valuation function $v$ is additive, i.e., $v$ can be decomposed
    as follows:
    \begin{align*}
      v(\x) = \sum_{i \in [m]} v_i x_i,
    \end{align*}
    where $v_i \in \R_+$.
  \end{assumption}

  With the additive valuation assumption, we refer each $v_i$ as the value of
  the $i$-th item. Moreover, we can make the following independent value
  assumption in addition.

  \begin{assumption}[Independent Values]\label{assump:independent}
    The prior distribution $\F$ is independent in each dimension and can be
    decomposed as $\F = F_1 \times \cdots \times F_m$, where each $v_i$ is
    independently drawn from $F_i$, i.e., $v_i \sim F_i$.
  \end{assumption}

  In the meanwhile, to show that our method is not limited to these assumptions,
  in \autoref{sec:exp}, we show how it can be applied to settings without these
  assumptions. In particular, with the help of the characterization results by
  \citet{daskalakis2013mechanism}, we are able to verify the optimality of the
  solution to an instance with correlated value distribution (while still with
  additive valuation functions).

\section{Problem Analysis}\label{sec:analysis}

  Although the revelation principle is widely adopted by the theoretical
  analysis of mechanism design problems to efficiently restrict the design
  spaces, we decided not to follow this approach when applying neural
  networks to solve such problems.

  The main difficulty of directly following the traditional revelation
  principle based approach is two-fold:
  \begin{itemize}
    \item It is unclear that what network structure can directly encode the
          incentive compatible \eqref{eq:ic} and individually rational
          \eqref{eq:ir} constraints;
    \item Some of the characterization results for additive valuation
          setting\footnote{Such as Myerson's virtual value for
          single-dimension and Rochet's increasing, convex and Lipschitz-1
          buyer utility function for multi-dimension \cite{dutting2019optimal}.}
          can be cast to certain network structures, but such structures are
          restricted (to additive valuation assumption) and heavily rely on the
          domain knowledge of the specific mechanism design problem.
  \end{itemize}

  In fact, the above difficulties also limit the generality of the methods
  built on these elegant but specific characterizations. For example, there
  might be some fundamental challenges while generalizing such approaches
  to the settings where the buyer is risk-averse (risk-seeking) or has
  partial (or bounded) rationality, etc. Furthermore, in many real applications,
  the buyer behavior models may come from real data instead of pure theoretical
  assumptions.

  To circumvent these difficulties and ensure the highest extendability, in
  this paper, we build up our method from the most basic {\em na\"ive
  mechanisms} --- simply let the buyer choose her favorite option --- which
  is even more close to the first principles of how people make decisions.
  Interestingly, via this approach, our method will automatically produce
  an exactly incentive compatible and individually rational mechanism. To the
  best of our knowledge, this is the first neural network based approach that
  outputs an both exactly incentive compatible and exactly individually rational
  mechanism under multi-dimensional settings.

  \subsection{Revisiting the Na\"ive Mechanism}\label{ssec:naive}

    We then briefly explain show how the na\"ive mechanism helps us to
    formulate a neural network based approach for mechanism design.

    Intuitively, the na\"ive mechanism in our context simply provides the
    buyer various menu items, i.e., allocations associated with different prices, and
    lets her choose the most prefered one. In this case, once a buyer utility
    function is specified (either by assumption or learnt from data), the
    choice of the buyer is simply an $\argmax$ of the utility function. As
    long as the utility function could be encoded via neural network, which
    is a mild assumption, the buyer's behavior model can encoded as a
    neural network with an additional $\argmax$ layer.\footnote{Even if the
    buyer utility function is not available, such a gadget could be replaced by
    any buyer behavior model (either given or learnt from data), which is
    encoded as a neural network.}

    \paragraph{High-level sketch of the network structure}
    For now, we can think the encoded mechanism as a black-box that outputs
    a set of allocation-payment pairs (see \autoref{sfig:naive}). These
    pairs then are feeded into many ``buyer networks'', each with different
    private valuation functions (hence different choices). Finally, the
    ``buyer networks'' output their choices and the choices are used to
    evaluate the expected objective of the mechanism designer, where the
    choices are weighted according to the probabilities of the corresponding
    private valuation functions and the training loss is simply the negative of
    the expected objective.

    \begin{figure}
      \subfigure[Na\"ive mechanism structure]
        {\includegraphics[height=0.23\textheight]{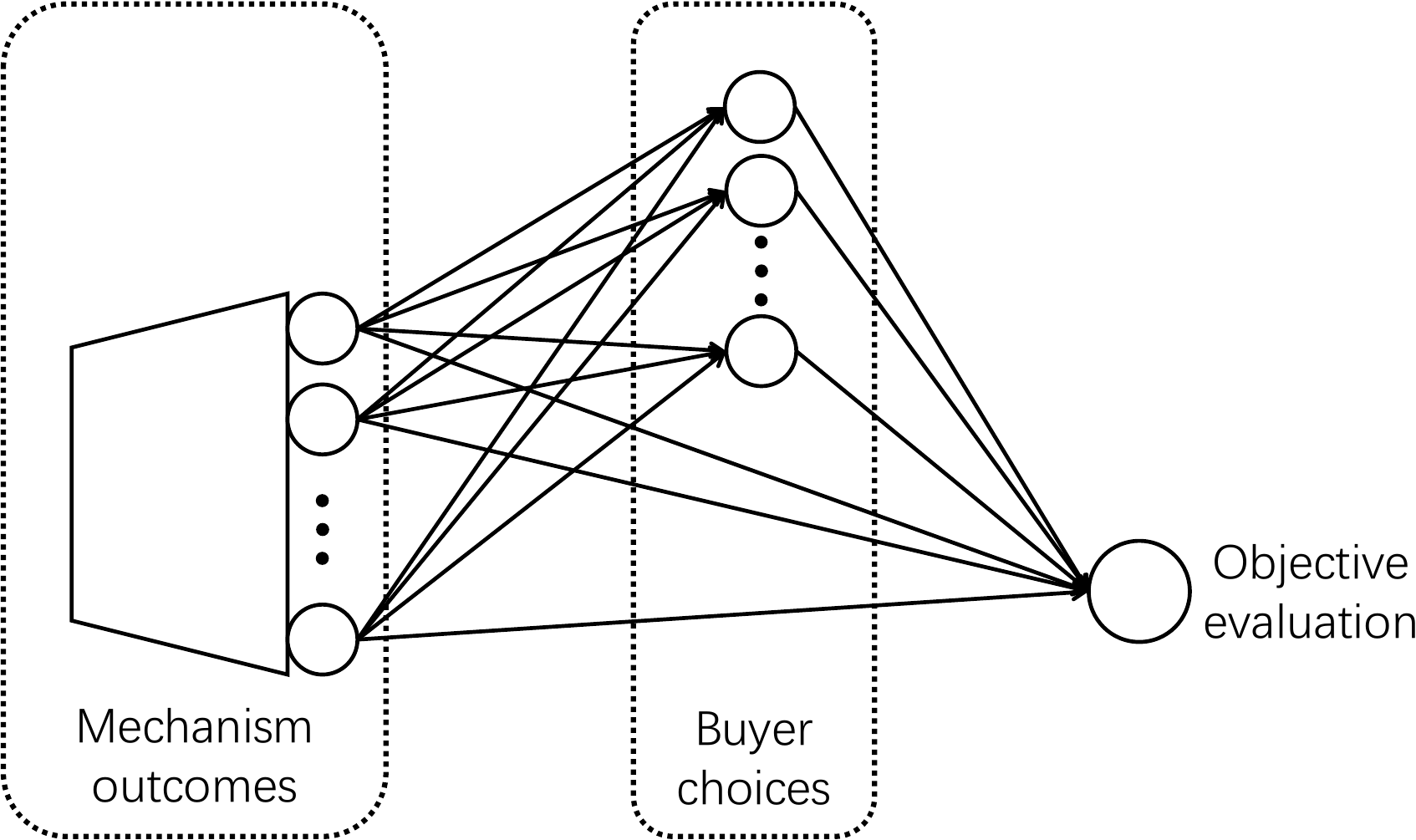}
        \label{sfig:naive}} %
      \hfill %
      \subfigure[Direct mechanism structure]
        {\includegraphics[height=0.23\textheight]{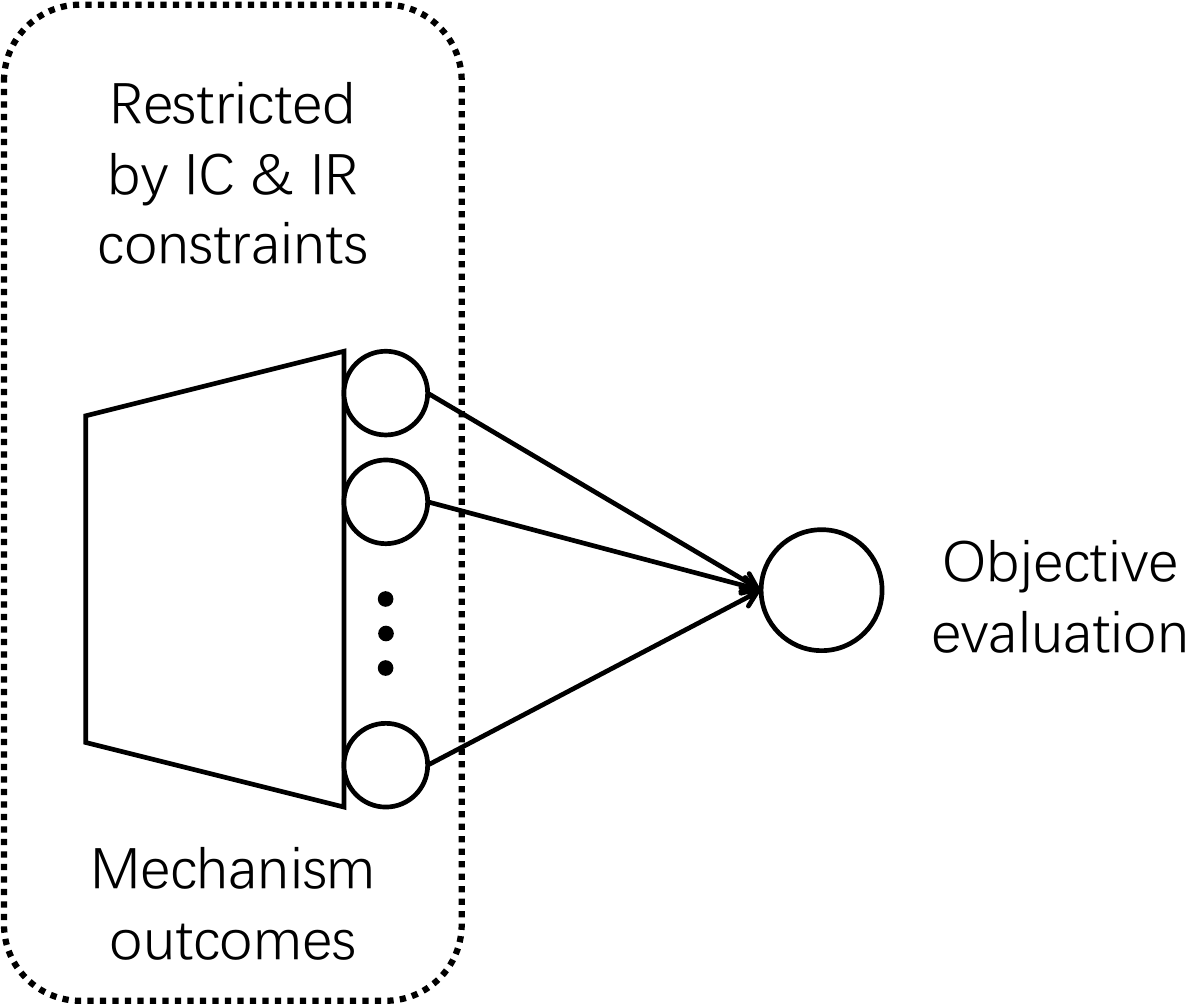}
        \label{sfig:direct}}
      \caption{A high-level abstraction of the neural networks.}
      \label{fig:hlstrc}
    \end{figure}

    One key advantage of formulating the network as a na\"ive mechanism
    rather than a direct mechanism is that no additional constraints (such
    as \ref{eq:ic} and \ref{eq:ir}) are required for the former. In fact,
    the difficulty of optimizing the direct mechanism network (see
    \autoref{sfig:direct}) is that the violations of \ref{eq:ic} or
    \ref{eq:ic} constraints are not directly reflected in the designer's
    objective. Hence the standard optimization methods for neural networks
    do not directly apply. In contrast, in the na\"ive mechanism network,
    the effect of any mechanism outcome mutations on the buyer preferences
    is reflected in the designer's objective via the ``buyer networks''.
    Such properties facilitate the optimization in standard training methods of
    neural networks.

  %

\section{Network Structure of \menunet}\label{sec:network}

  Our \menunet structure contains two networks: the mechanism network and the buyer network. Since the networks represent a na\"ive mechanism, the output of the mechanism network is a set of choices along with different prices (or menu items) and the buyer network takes the set of menu items as input and outputs its choice. The overall network structure is shown in \autoref{fig:overall_net}.

  \begin{figure}
    \centering
  	\includegraphics[width=0.65\textwidth]{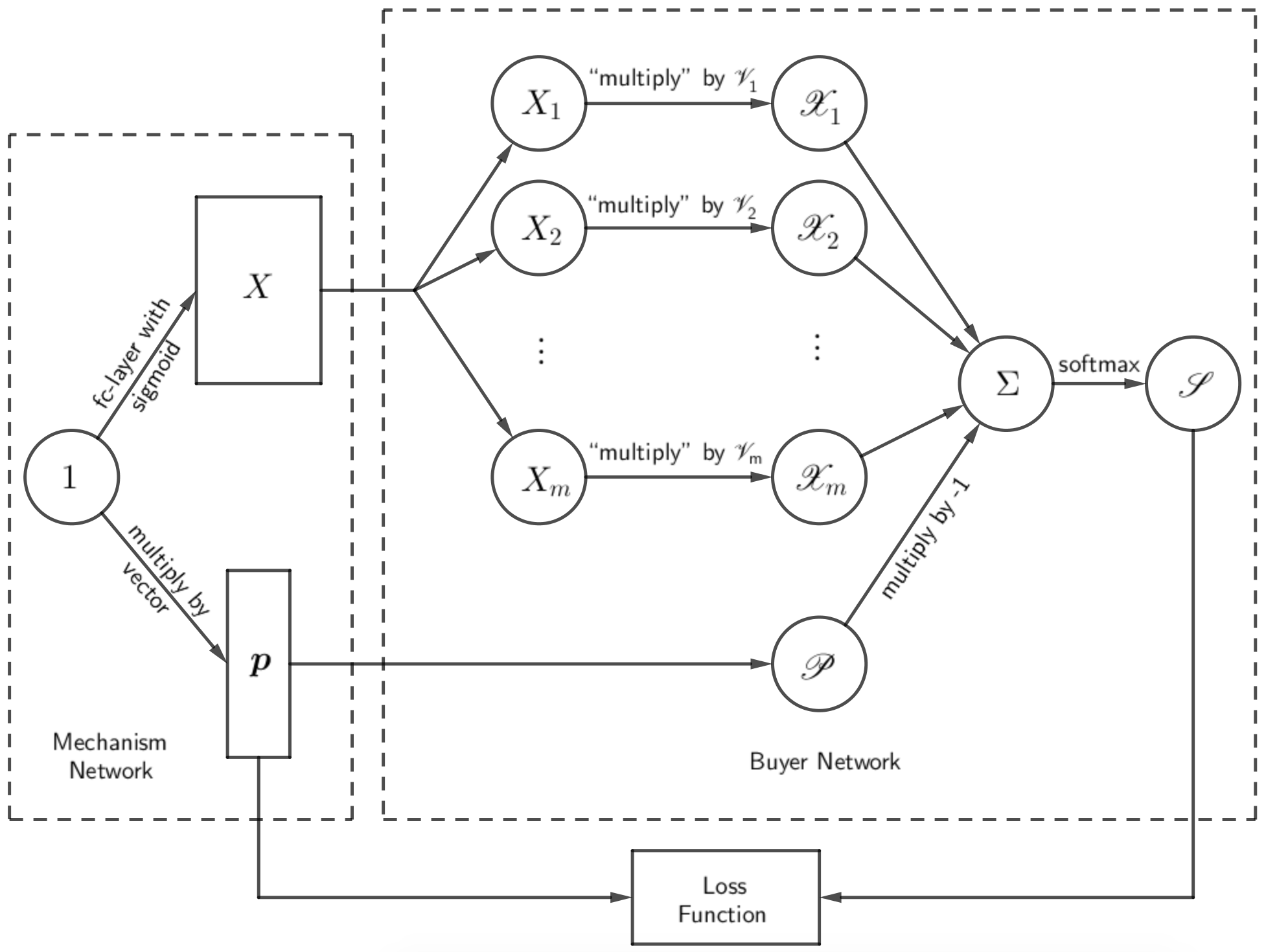}
  	\caption{\menunet structure, in which the buyer network corresponds to a rational buyer maximizing its utility. The illustrated example is for quasi-linear utilities. In other cases, the buyer network can be constructed according to his utility function, or other networks trained from interaction data.}
  	\label{fig:overall_net}
  \end{figure}

  \subsection{Mechanism Network}\label{ssec:mnet}
  In most applications, a neural network usually takes a possible input $x$ and then outputs a possible output $y$. However, our mechanism network is different from most neural networks in the sense that its output is a set of menu items, which already represents the entire mechanism. Therefore, our mechanism network does not actually need to take an input to give an output.

  However, in order to fit in with most neural network frameworks, we use a one dimensional constant 1 as the input of our mechanism network. The output of the network consists of two parts. The first part is an allocation matrix $X$ of $m$ rows and $k$ columns, where $m$ is the number of items and $k$ is the number of menu items. Each column of the allocation matrix contains the allocation of all $m$ item. The second part is a payment vector $\p$ of length $k$, representing $k$ different prices for the $k$ menu items. The last column of the allocation matrix and the last element of the payment vector is always set to be 0. This encodes the ``exit'' choice of the buyer and ensures that the buyer can always choose this menu item to guarantee individual rationality.

  The structure of the mechanism network is simple enough. The constant input 1 goes through a 1 fully connected layer to form each row $X_i$ (except the last column, which is always 0) of the allocation matrix. We choose the sigmoid function as the activation function since the allocation of each item is always inside the interval $[0, 1]$. The payment vector is even simpler. Each element $p_i$ of the payment vector is formed by multiplying the input constant by a scalar parameter. Therefore, the training of our network is very fast, since the network structure is very simple.

  \subsection{Buyer Network}\label{ssec:bnet}

  The buyer network is a function that maps a mechanism to the buyer's strategy $\s(v)$ (a distribution over all possible menu items) for each value profile $v=(v_1, v_2, \dots, v_m)$, where each $v_i$ is the value of getting the $i$-th item.  The output of the mechanism network (the allocation matrix $X$ and the payment vector $\p$) is taken as the input of the buyer network. To define the output of the buyer network, suppose that each $v_i$ is bounded and $0\le v_i\le \bar{v_i}$. We discretize the interval $[0, \bar{v_i}]$ to $d_i$ discrete values. Let $V_i$ be the set of possible discrete values of $v_i$ and define $V=\prod_{i\in [m]}V_i$.

  The output of the buyer network is a $m+1$ dimensional tensor, with the first $m$ dimension corresponding the buyer's $m$ dimensional value, and the last dimension representing the probability of choosing each menu item. Therefore, the $i$-th ($i\le m$) dimension of the tensor has length $d_i$ and the last dimension has length $k$.

  Although here we use the same notation as in \autoref{assump:additive}, this notation does not lose generality since we do not make any assumption about the buyer's valuation of obtaining multiple items or only a fraction of an item. It is also worth mentioning that the buyer's utility function is not necessary to build the buyer network, since the network only outputs buyer's strategy, which may not even be consistent with any utility function.

  The buyer network can be any type of network that has the same format of input and output as described above. When we do not know the buyer's exact utility function but have plenty of interaction data (e.g., the sponsored search setting), we can train the buyer network with the the interaction data.

  When the buyer's utility function is known, we can manually design the buyer network structure so that the network outputs the buyer's strategy more accurately. For example, when \autoref{assump:additive} and \autoref{assump:independent} holds, we know that the buyer always choose the menu item that maximizes his additive valuation with probability 1. We can construct $m$ tensors $\mathscr{V}_1, \mathscr{V}_2,\dots,\mathscr{V}_m$,  with size $d_1\times d_2\times \cdots \times d_m$. In $\mathscr{V}_i$, an element's value is only determined by its $i$-th dimensional index in the tensor, and it equals the $j$-th discretized value of the interval $[0, \bar{v_i}]$, if its $i$-th dimensional index is $j$. Recall that the $i$-th row of the allocation matrix $X_i$ represents different allocations of the $i$-th item in different menu items. We then multiply the $i$-th tensor with the $X_i$ to get an $m+1$ dimensional tensor $\mathscr{X}_i$ with size $d_1\times d_2\times \cdots \times d_m\times k$.

  We also construct a payment tensor $\mathscr{P}$ with size $d_1\times d_2\times \cdots \times d_m\times k$, where an element equals to the $p_i$ if its index for the last dimension is $j$.

  Finally, we compute the utility tensor $\mathscr{U}$ by
  \begin{gather*}
  \mathscr{U}=\left(\sum_{i \in [m]}\mathscr{X}_i\right) - \mathscr{P}.
  \end{gather*}
  And then apply the softmax function to the last dimension of the utility tensor $\mathscr{U}$ to produce the output $\mathscr{S}$, which is an aggregation of $\s(v), \forall v\in V$. One can easily verify that for each value profile, the menu with the largest utility has the highest probability of being chosen. Of course, we also multiply the utility tensor by a large constant to make the probability of the best menu item close enough to 1.

  \subsection{Loss Function}\label{ssec:loss}

  The loss function can be any function specified according the  mechanism designer's objective. However, in this paper, we mainly focus on how to optimize the revenue of the mechanism and set the loss function to be the negative revenue.

  Recall that the output of the buyer network is the buyer's strategy $\s(v)$ for each value profile $v$. Then the loss function of the networks is
  \begin{gather*}
  \Loss=-\Rev=-\sum_{v\in V}\mathrm{Pr}\left[v\right]\p^T\s(v)
  \end{gather*}
  where $\mathrm{Pr}\left[v\right]$ is the probability that $v$ appears, which can be easily computed from the joint value distribution $\F$.

  Note that in the above loss function, we do not make any assumption about the probability distribution $\mathrm{Pr}[v]$. Our networks are able to handle any joint distribution, including correlated ones.


\section{Experiments and Analysis}\label{sec:exp}
In this section, we first list some results of our neural networks in Section \ref{ssec:exp}. Inspired by these results, we are able to prove the closed-form optimal mechanisms in some cases where the exact optimal solutions are previously unknown: i) the setting with correlated triangle distribution (\autoref{ssec:tri}) and ii) the setting with uniform square distribution but restricted menu size (\autoref{ssec:limitms}). The theoretical analysis and proofs of our newly discovered optimal mechanisms are presented in \autoref{ssec:thm}.

To the best of our knowledge, we are the first to discover exact optimal mechanisms under the help with the neural network based approach. Although one still need to tolerate the complexity of the theoretical proof (mostly on constructing the matching dual solution), the neural network can greatly help on guessing the structure of the optimal primal solution. We believe the methodology  is of its own interests. As one followup, in a recent version, \citet{dutting2020optimal} followed this methodology and discovered the optimal mechanisms for some different correlated triangle distributions.

\subsection{Experiment results}\label{ssec:exp}
  \subsubsection{Uniform $[0, c] \times [0, 1]$}\label{ssec:unic1}
  The optimal mechanism for this setting is already known \cite{thirumulanathan2016optimal}. We draw both the optimal mechanism and our experiments results together in Figure \ref{fig:empunic1}. The color blocks represents the mechanism given by our network, where each color corresponds to a different menu item. The dashed line represents the optimal mechanism (they are NOT drawn according to the color blocks). The two mechanisms are almost identical except for the slight difference in Figure \ref{sfig:empc=1.9}.
    \begin{figure}[t]

      \subfigure[$c = 1.5$]
        {\includegraphics[height=0.135\textheight]{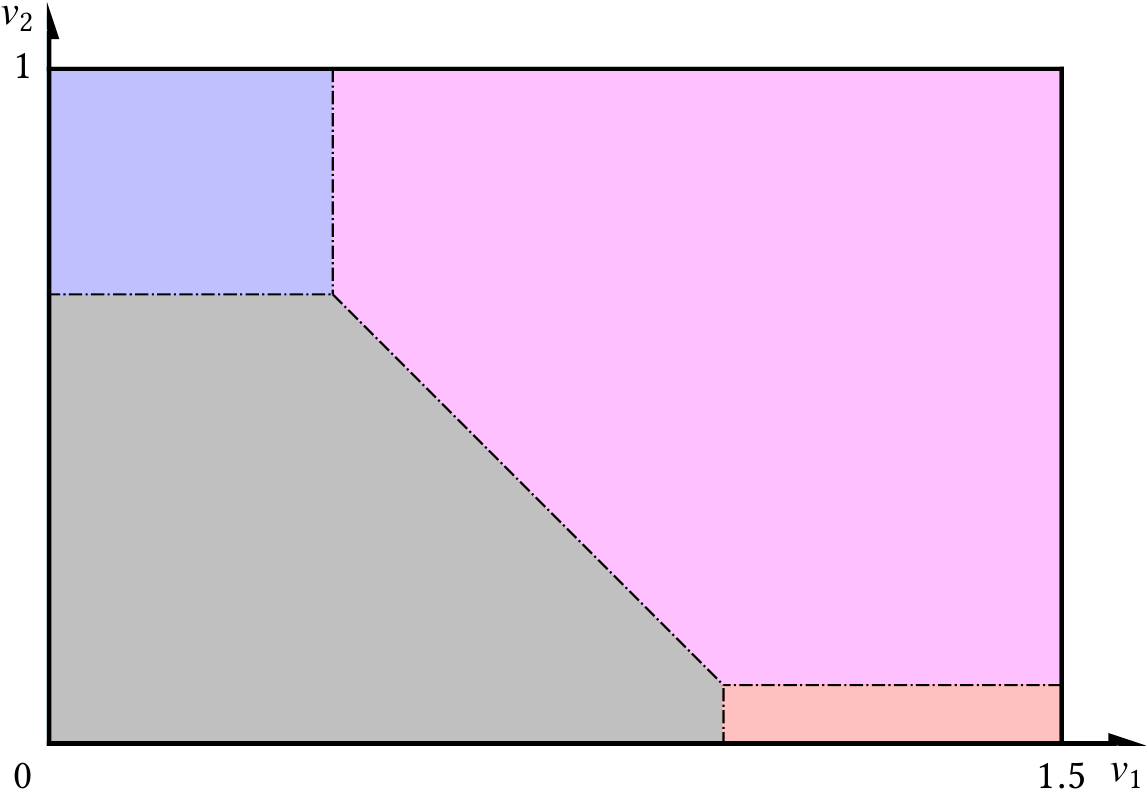}
        \label{sfig:empc=1.5}}
      \subfigure[$c = 2.5$]
        {\includegraphics[height=0.135\textheight]{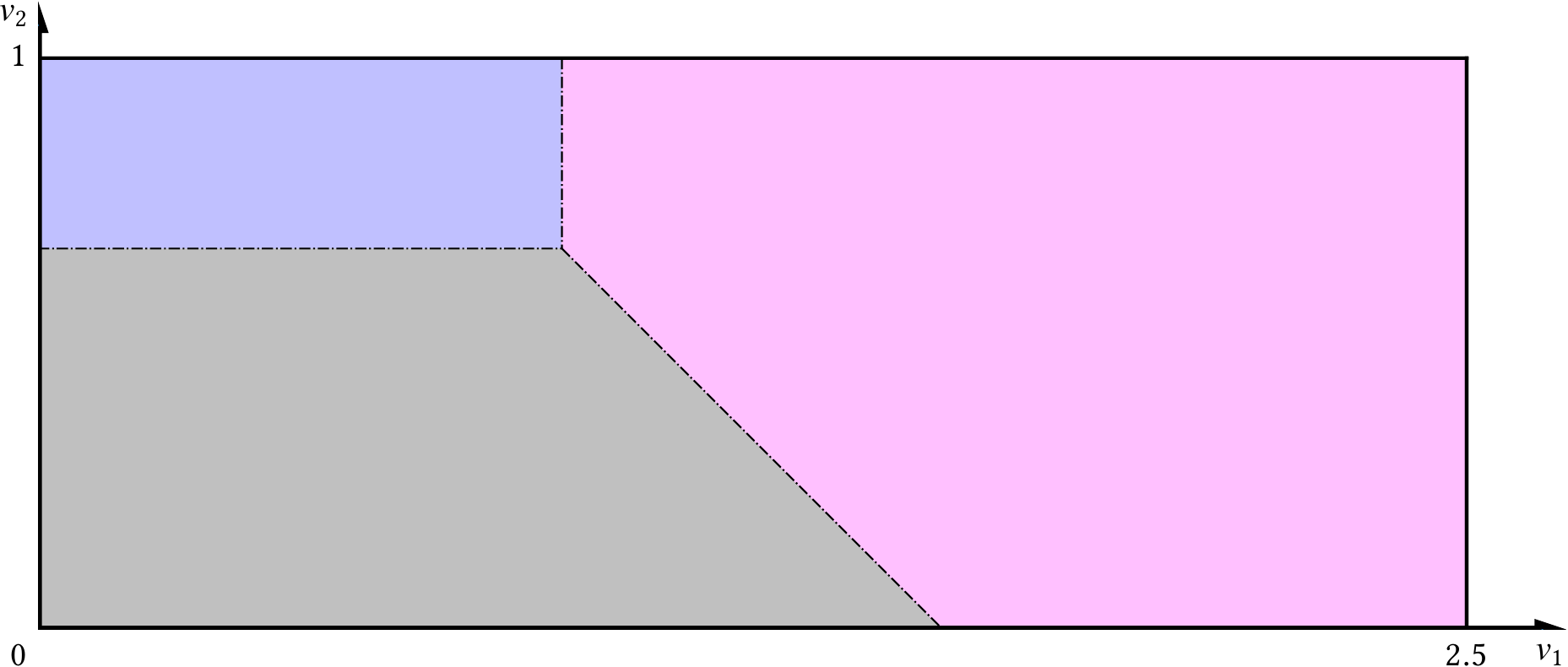}
        \label{sfig:empc=2.5}}
      \subfigure[$c = 1.9$]
        {\includegraphics[height=0.135\textheight]{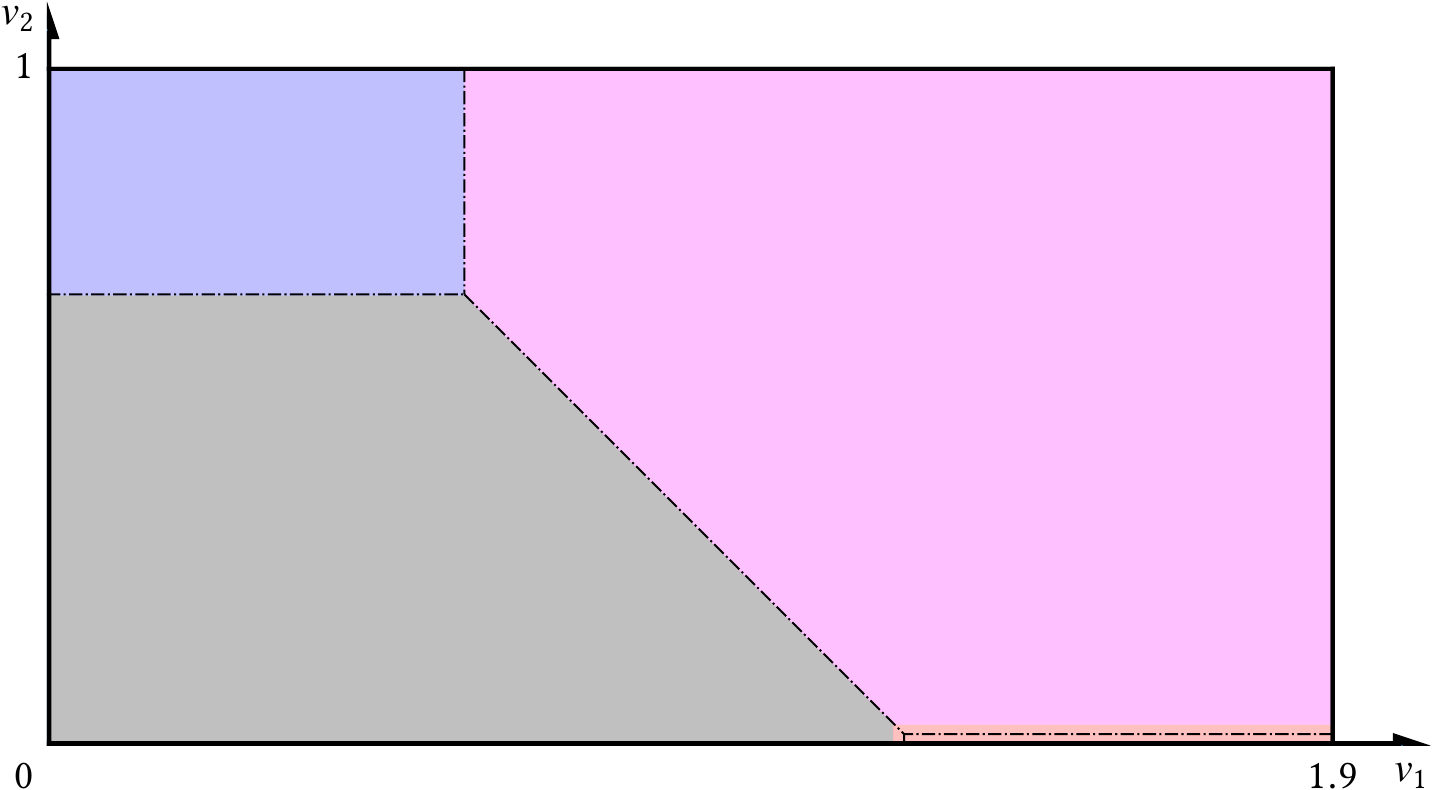}
        \label{sfig:empc=1.9}} %
      \hfill %
      \subfigure[$c = 2$]
        {\includegraphics[height=0.135\textheight]{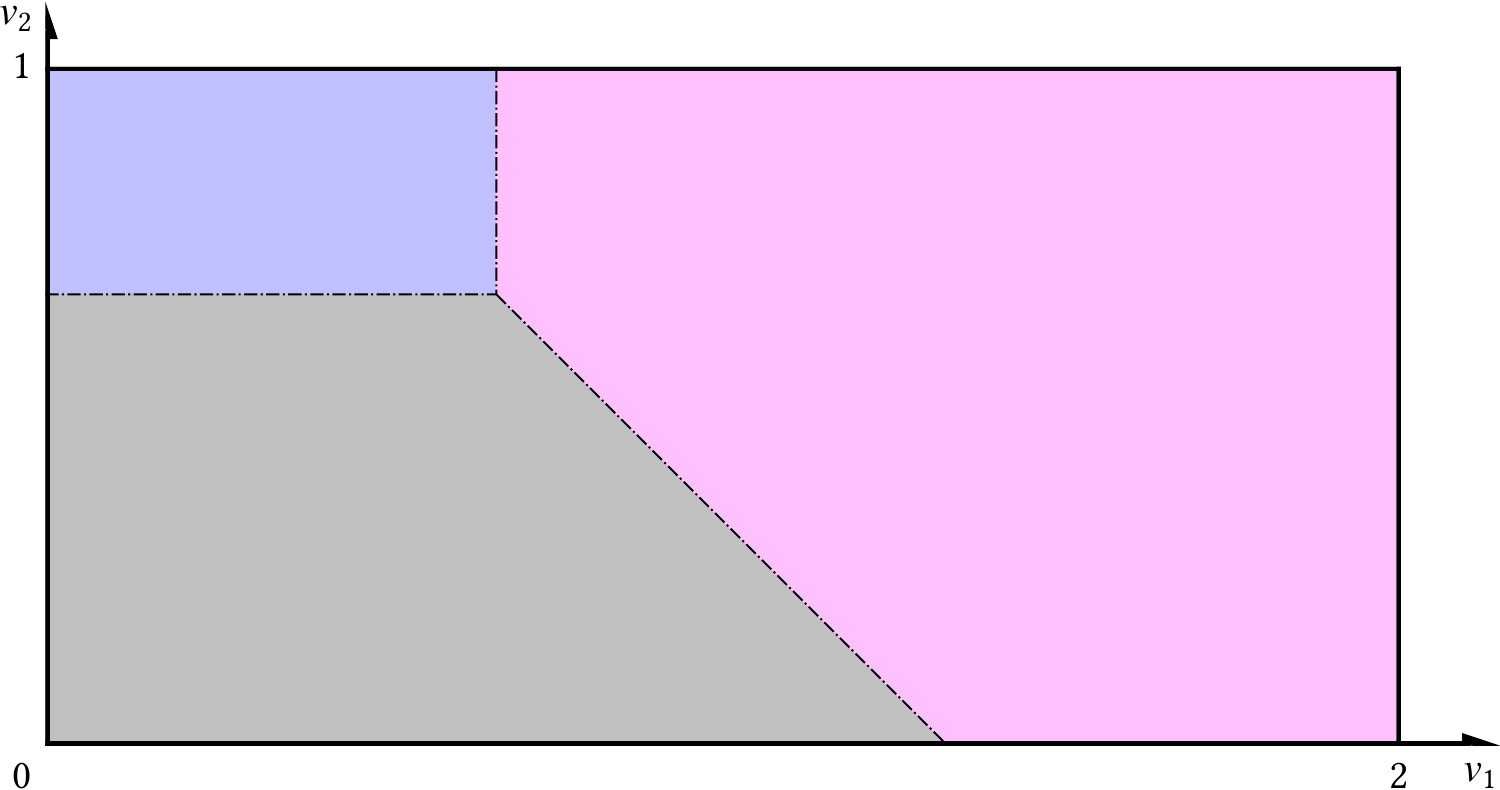}
        \label{sfig:empc=2}} %
      \caption{Comparison between computed solutions and optimal solutions.}
      \label{fig:empunic1}
    \end{figure}

  \subsubsection{Correlated Distribution: Uniform Triangle}\label{ssec:tri}
  Suppose that the buyer's value $v=(v_1,v_2)$ is uniformly distributed among the triangle described by $\frac{v_1}{c}+v_2\le 1, v_1\ge  0, v_2\ge 0$, where $c\ge1$. The color blocks in Figure \ref{fig:empunitric2} show the mechanisms given by our network. Note that in our framework, the joint value distribution is only used to compute the objective function. So our framework can handle arbitrary value distributions.
  \begin{figure}[t]
    \phantom{1} \hfill
  	\subfigure[$c = 1.25$]
  	{\includegraphics[height=0.165\textheight]{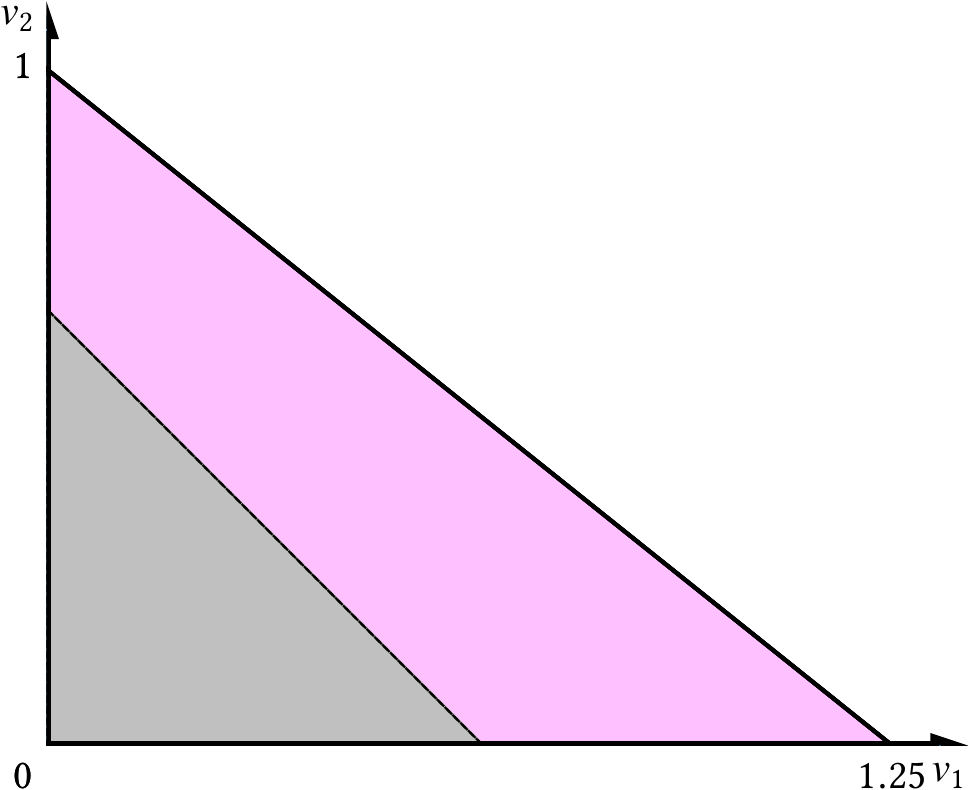}
  		\label{sfig:emptric=1.25}} %
  	\hfill
  	\subfigure[$c = 2$]
  	{\includegraphics[height=0.165\textheight]{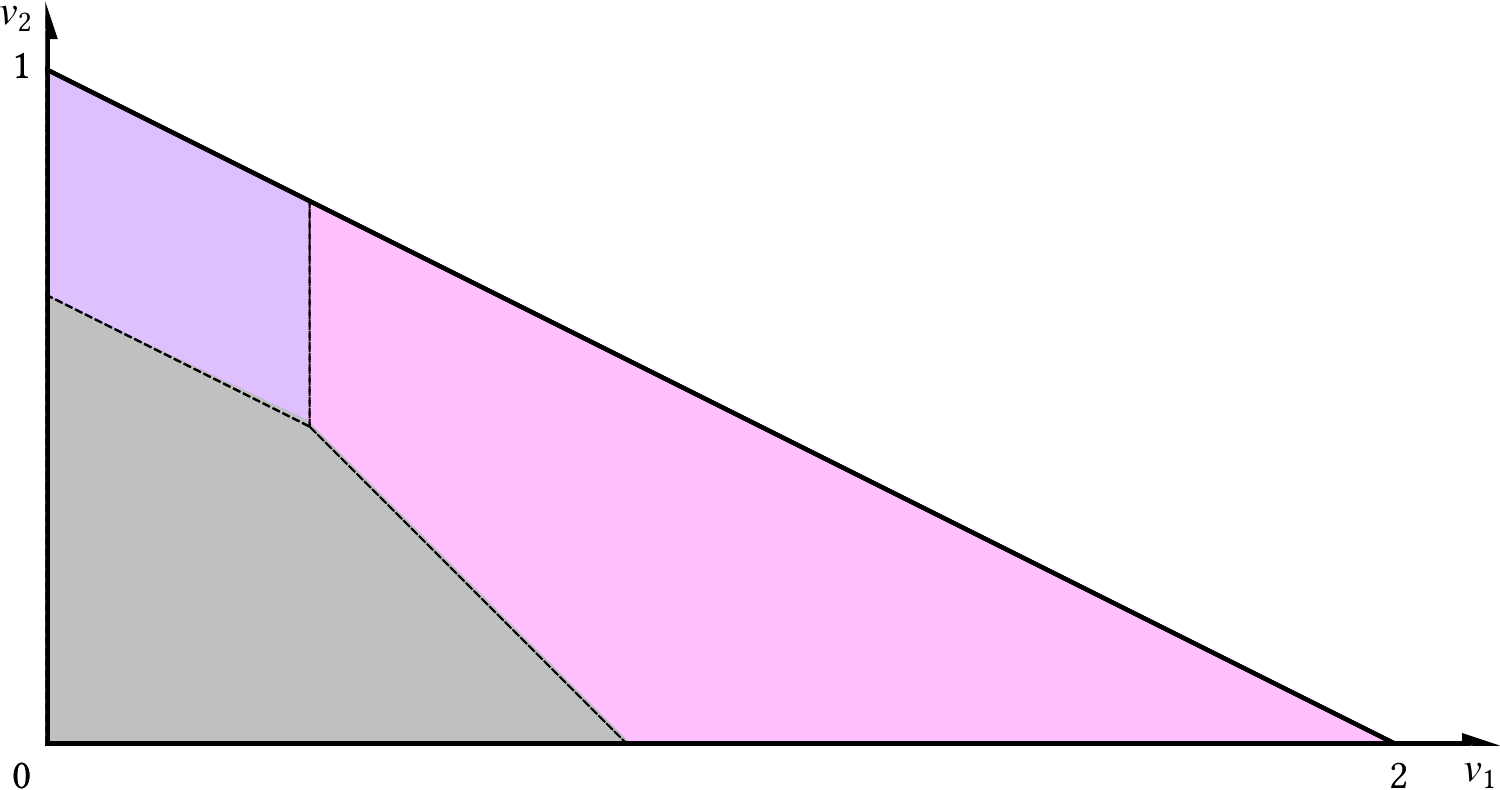}
  		\label{sfig:emptric=2}} %
      \hfill \phantom{1}
  	\caption{Uniform Triangle.}
  	\label{fig:empunitric2}
  \end{figure}

  In fact, guided by these experiment results, we are able to find the closed-form optimal mechanism for this kind of value distributions. In particular, there are two possible cases for this problem. When $c$ is large, the optimal mechanism contains two menu items. And when $c$ is small, the optimal contains only two menus, i.e., use a posted price for the bundle of the items. Formally, we have

  \begin{theorem}\label{thm:triangle_mech}
  	When $c > \frac{4}{3}$, the optimal menu for the uniform triangle distribution contains the following items:
      $(0,0),0$,
      $(\frac{1}{c}, 1),\frac{2}{3}$, and
      $(1,1),\frac{2}{3}c-\frac{1}{3}\sqrt{c(c-1)}$.

  When $c \le \frac{4}{3}$, the optimal menu for the uniform triangle distribution contains the following items:
    $(0,0),0$ and
    $(1,1),\sqrt{\frac{c}{3}}$.
  \end{theorem}
The proof is deferred to \autoref{ssec:thm}. In a recent version, \citet{dutting2020optimal} followed our approach and also gave the optimal mechanisms for similar triangle distributions. But the support of the distribution in their case is different from ours by a constant translation.

  \subsubsection{Restricted Menu Size}\label{ssec:limitms}
  The output of our mechanism network is a set of menus. Thus we can control the menu size by directly setting the output size of the network.

  Restricting the menu size results in simpler mechanisms. It is known that the optimal menu for some distributions contains infinitely many items \cite{daskalakis2013mechanism}. Such results directly motivates the study of simple mechanisms, since they are easier to implement and optimize in practice.

  We consider the case where the buyer's value is uniformly distributed in the unit square $[0,1]^2$. It is known that the optimal mechanism contains 4 menu items. When the menu can only contain at most 2 items, the optimal mechanism is to trivially set a posted price for the bundle. The experiment results are shown in Figure \ref{fig:empunirms}.
  \begin{figure}[t]
    \phantom{1} \hfill
  	\subfigure[At most $2$ menus.]
  	{\includegraphics[height=0.155\textheight]{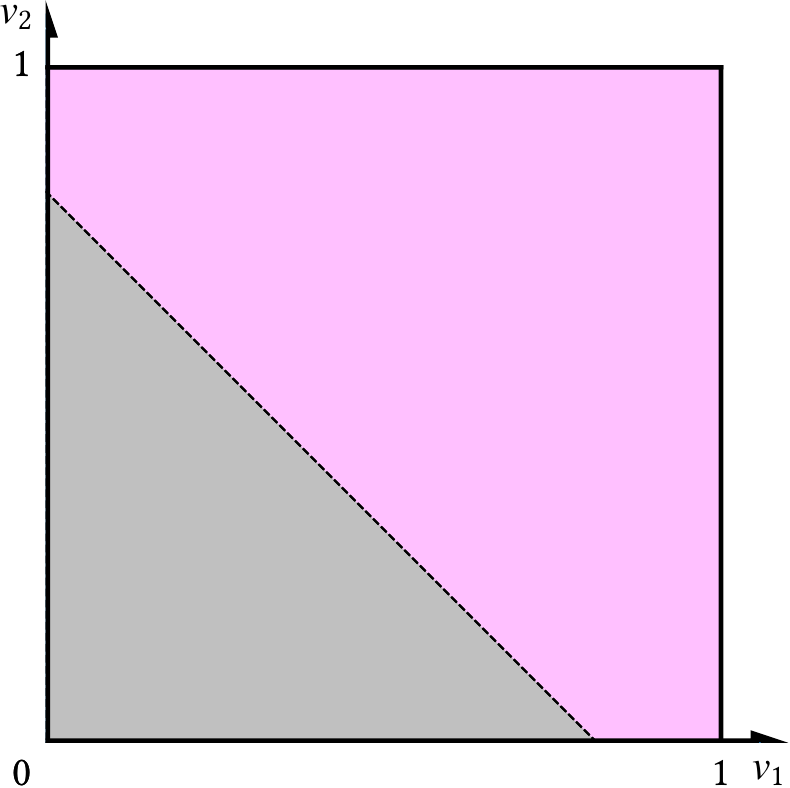}
  		\label{sfig:empunirms2}} %
  	\hfill
  	\subfigure[At most $3$ menus.]
  	{\includegraphics[height=0.155\textheight]{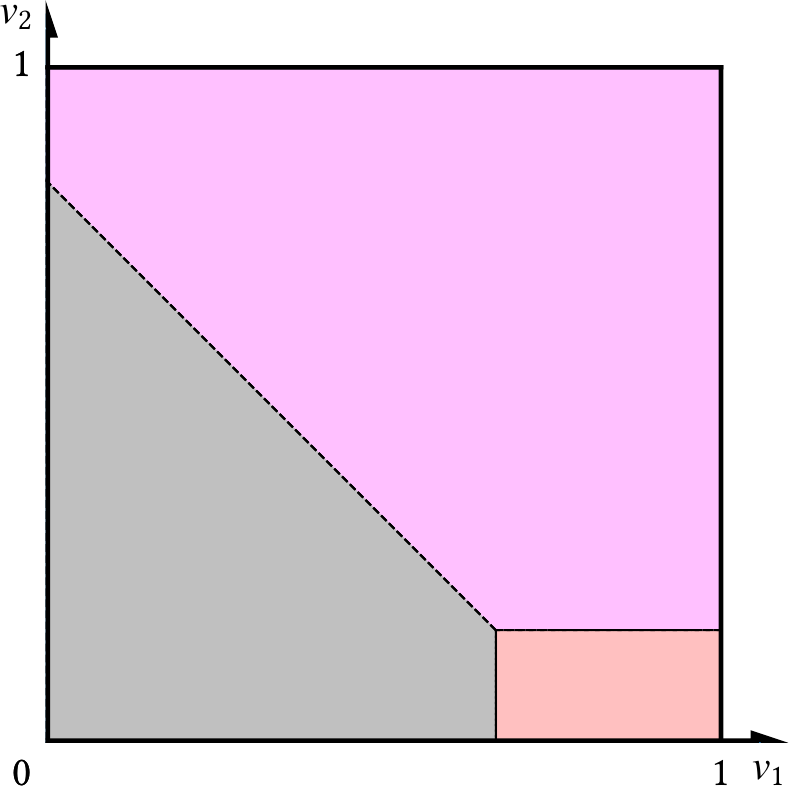}
  		\label{sfig:empunirms3}} %
    \hfill \phantom{1}
  	\caption{Uniform $[0, 1]^2$ with restricted menu size.}
  	\label{fig:empunirms}
  \end{figure}

  Surprisingly, when the menu can have at most 3 items, our network gives an asymmetric menu, despite that the value distribution is symmetric. In fact, we can also find the optimal menu with at most 3 items analytically. Our analysis shows that the optimal menu is indeed asymmetric. The intuition is that, if we add a symmetry constraint to the solution, then the optimal menu degenerates to a 2-item one. We provide the theoretical result here, but defer the proof to Section \ref{ssec:thm}.

    \begin{theorem}\label{thm:3menu_informal}
  	The optimal at-most-three-menu mechanism for two additive items with $\v
  	\sim U[0, 1]^2$ is to sell the first item at price $2 / 3$ or the bundle of
  	two items at price $5 / 6$, yielding revenue $59 / 108 \approx 0.546296$.

  	By symmetry, the mechanism could also be selling the second item at price
  	$2 / 3$ or the bundle of two items at price $5 / 6$. In particular, these
  	is no other at-most-three-menu mechanisms could generate as much revenue as
  	they do.
  \end{theorem}

  \subsubsection{Unit-Demand Buyer}\label{ssec:unit}
  The unit-demand setting is also intensively studied in the literature. In this
  setting, the allocation must satisfy $x_1+x_2\le 1$. \cite{thirumulanathan2017optimal} provides
  detailed analysis and closed-form solutions on the unit-demand setting. With
  slight modifications, our mechanism network can also produce feasible
  allocations in this setting. Instead of applying the sigmoid function to each
  element of the allocation matrix, we apply a softmax function to each column
  (representing each menu item) of the allocation matrix. However, with such a
  modification, the allocation satisfies $x_1+x_2= 1$ rather than $x_1+x_2\le 1$.
  The solution is to add an extra dummy element to each column before applying
  the softmax function.

  The experiment results are shown in \autoref{sfig:unidemand}.

  \begin{figure}[!h]
  	\subfigure[Unit demand.]
  	{\includegraphics[height=0.155\textheight]{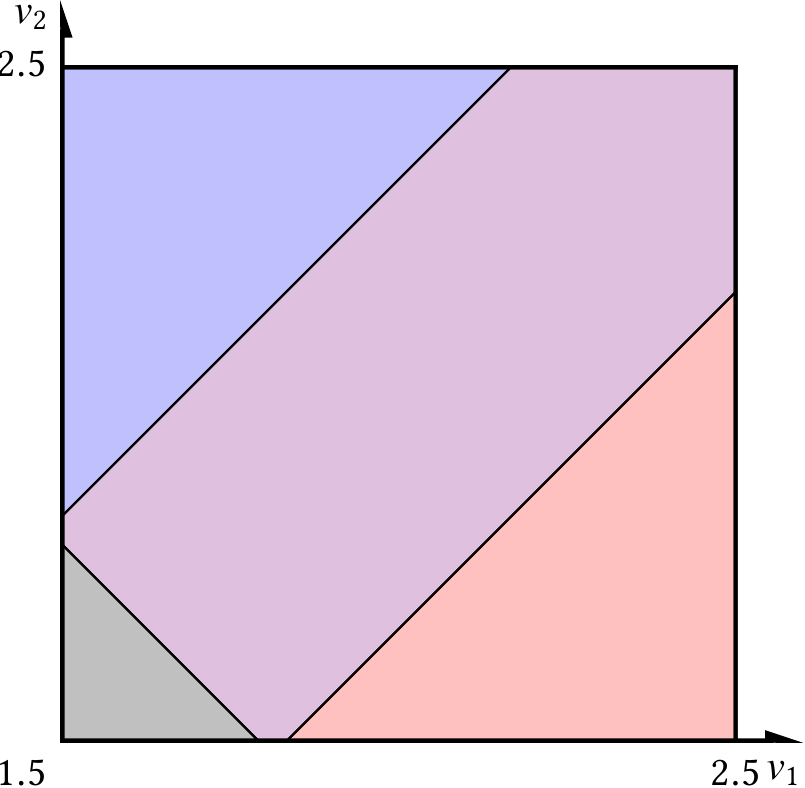}
  		\label{sfig:unidemand}} %
  	\hfill
  	\subfigure[Combinatorial Value.]
  	{\includegraphics[height=0.155\textheight]{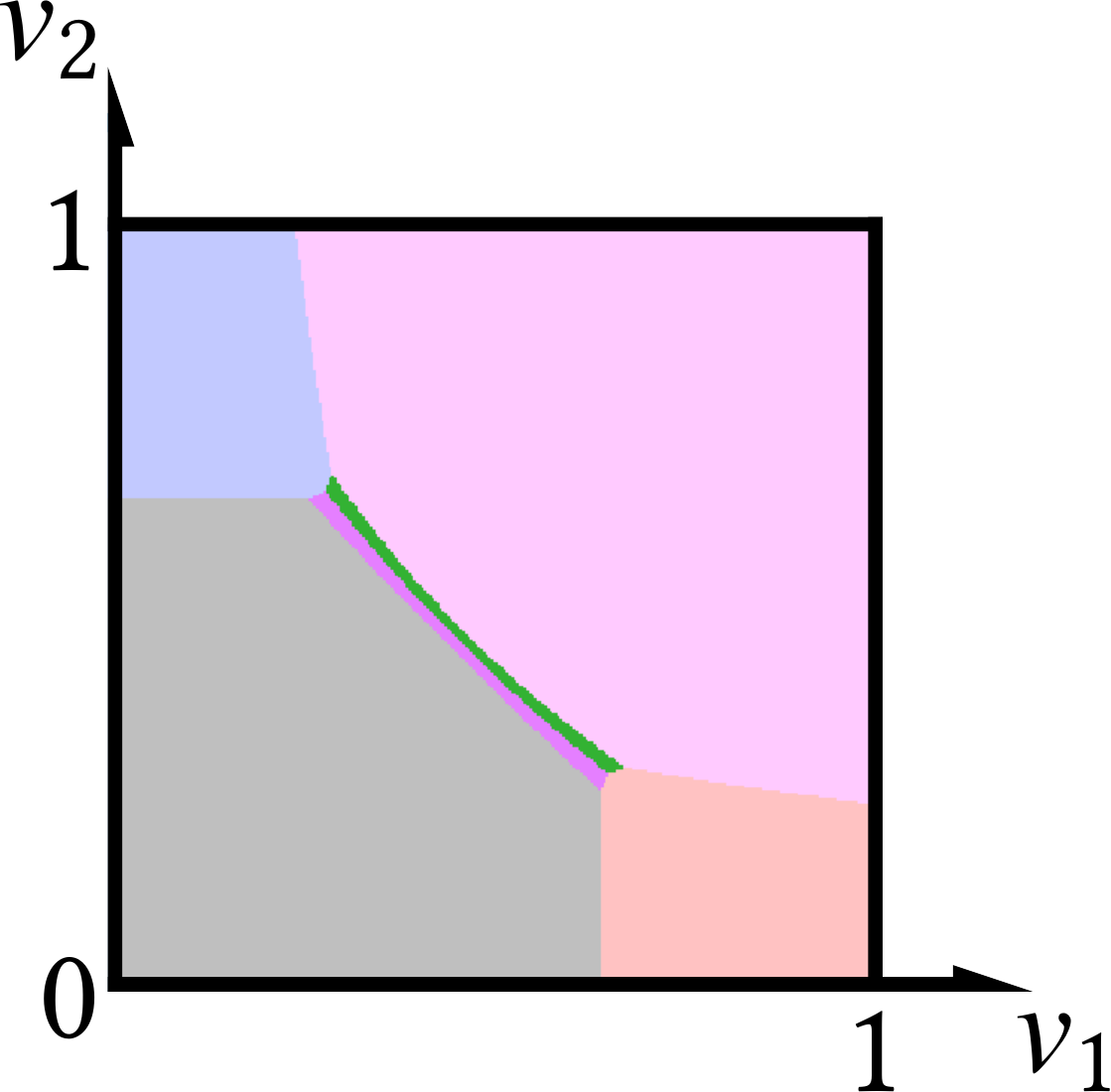}
  		\label{sfig:combin}} %
  	\hfill
  	\subfigure[Deterministic allocation.]
  	{\includegraphics[height=0.155\textheight]{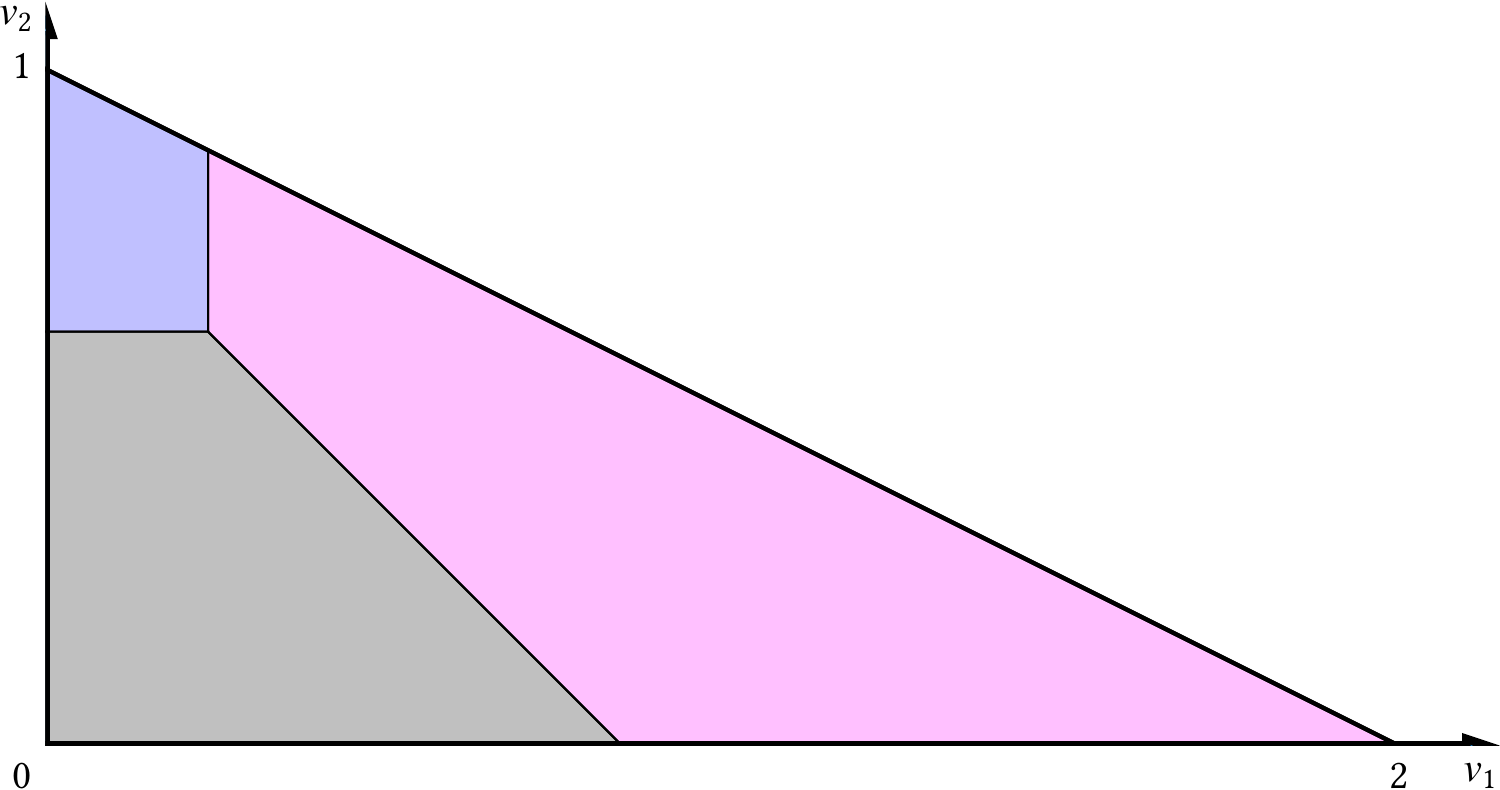}
  		\label{sfig:deter}} %
  	\caption{Empirical results.}
  	\label{fig:ucd}
  \end{figure}

  \subsubsection{Combinatorial Value}\label{ssec:comb}
  Our framework structure can also handle the case where the buyer has
  combinatorial values. The following Figure \ref{sfig:combin} shows mechanism given by our
  network for a buyer with $u(v_1,v_2)=x_1v_1+x_2v_2+x_1v_1x_2v_2-p$. In this case, we
  need to slightly modify the buyer network by adding the extra $x_1v_1x_2v_2$ term,
  which can be easily implemented.

  \subsubsection{Deterministic Mechanisms}\label{ssec:deterministic}

  We can use our networks to find the optimal deterministic mechanisms for any
  joint value distributions. Similar to the restricted menu size case,
  deterministic mechanisms are also important in practice, since they are easy
  to understand and implement. In this case, the mechanism network can be
  further simplified, since for selling 2 items, there can only be 4 possible
  deterministic menu items, with allocations $(0,0),(0,1),(1,0),(1,1)$.
  Therefore, the only parameters in the mechanism network are the corresponding
  prices.

  \autoref{sfig:deter} shows our experiment results on uniform distributions
  among the triangle described in \autoref{ssec:tri}. According to
  \autoref{thm:triangle_mech}, the optimal mechanism is not deterministic when
  $c=2$, Our experiments show that such a constraint decreases the revenue by
  $0.14\%$.

%
%
%

\subsection{Theoretically Provable Optimal Mechanisms}\label{ssec:thm}

  In this section, we provide theoretical proofs for some of the findings via
  our neural network. To the best of our knowledge, these results are previously
  unknown.


  \subsubsection{Optimal mechanisms for selling two items with correlated distributions}
  As described in Section \ref{ssec:tri}, there are two possible cases for the optimal mechanism when the buyer's value is uniformly distributed among the triangle. The solutions are shown in Figure \ref{fig:correlated}.
  \begin{figure}[t]
    \phantom{1} \hfill %
  	\subfigure[When $c> \frac{4}{3}$]
  	{\includegraphics[height=0.16\textheight]{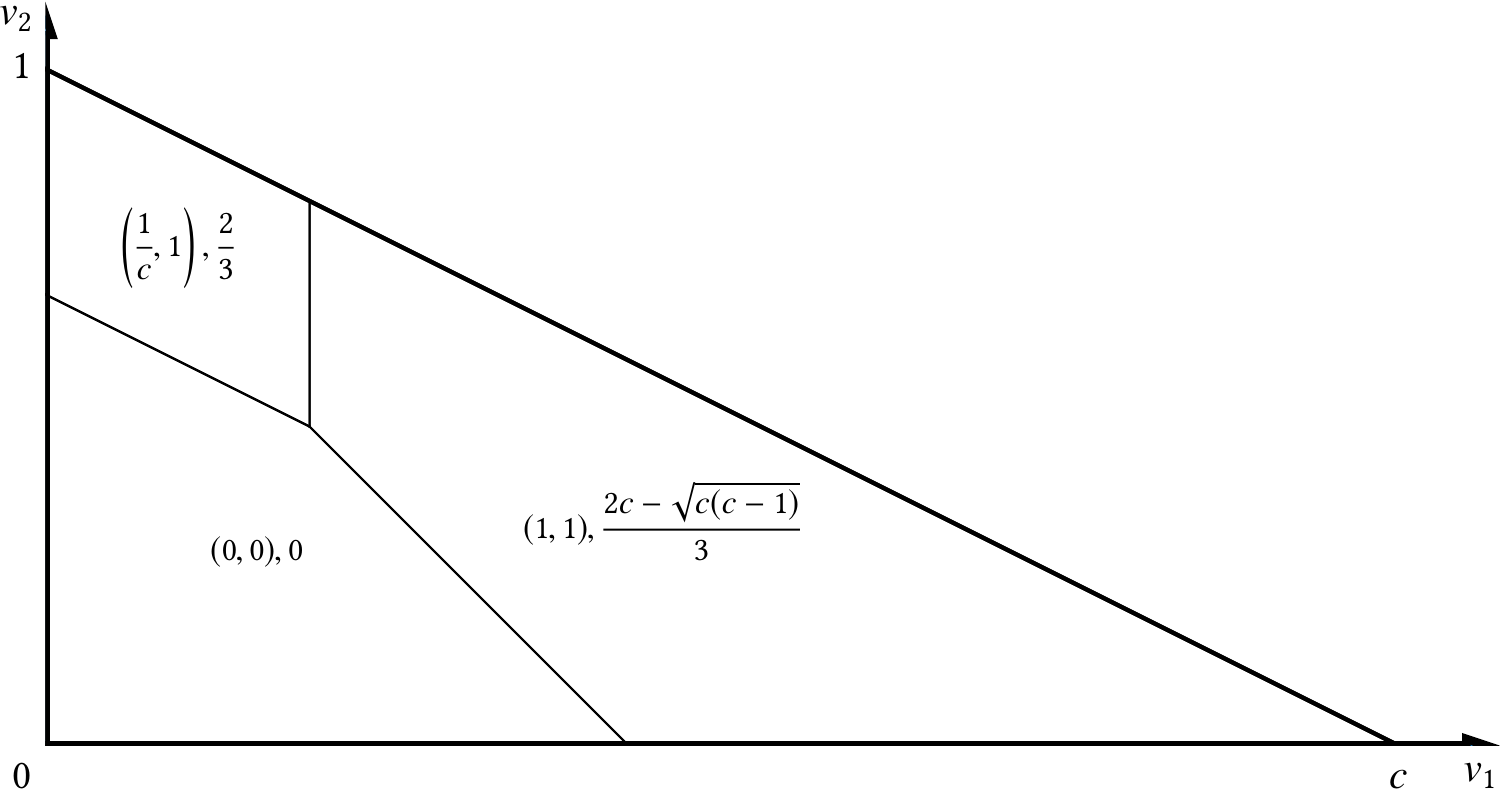}
  		\label{sfig:tirangle_opt_3}} %
  	\hfill
  	\subfigure[When $c\le\frac{4}{3}$]
  	{\includegraphics[height=0.16\textheight]{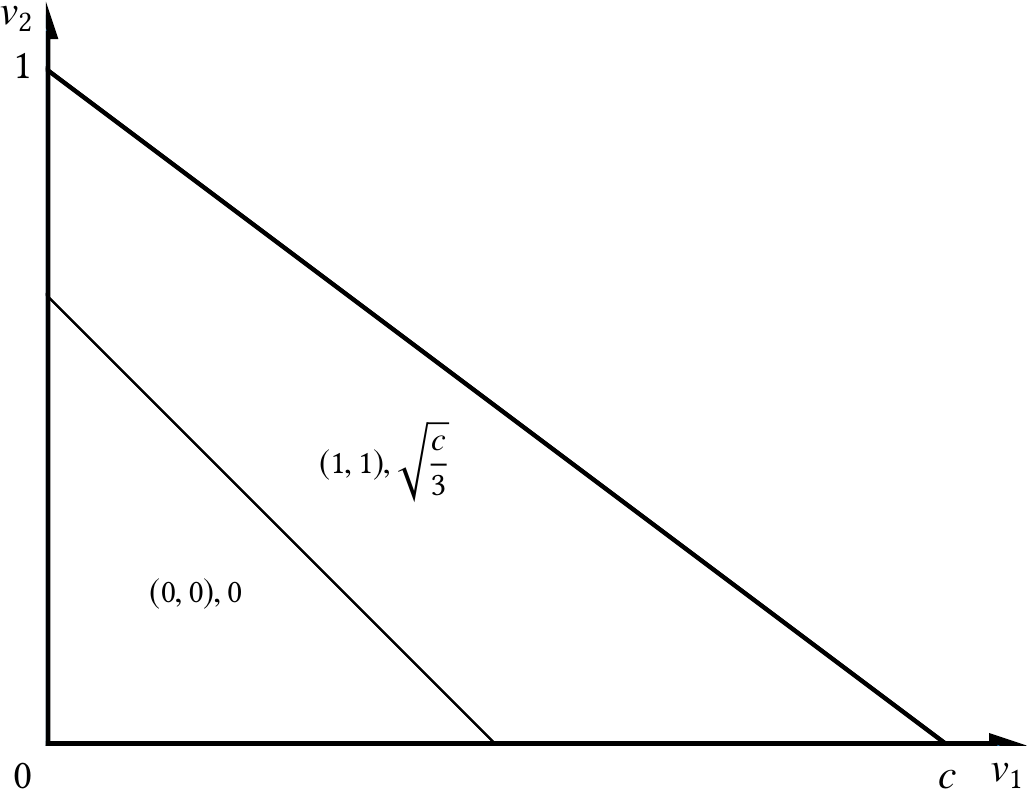}
  		\label{sfig:tirangle_opt_3=2}} %
    \hfill  \phantom{1}  %
  	\caption{Uniform Triangle.}
  	\label{fig:correlated}
  \end{figure}
  We solve the problem case by case.
  \begin{theorem}
  	\label{thm:triangle}
  	For any $c > \frac{4}{3}$, suppose that the buyer's type is uniformly distributed among the set $T=\{(v_1,v_2)~|~\frac{v_1}{c}+v_2\le 1, v_1\ge  0, v_2\ge 0 \}$. Then the optimal menu contains the following items:
      $(0,0),0$,
      $(\frac{1}{c}, 1),\frac{2}{3}$, and
      $(1,1),\frac{2}{3}c-\frac{1}{3}\sqrt{c(c-1)}$.
  \end{theorem}
  \begin{Remark}\label{remark:swr}
    Note that the condition $c>\frac{4}{3}$ guarantees that the price of the third menu item is positive.
  \end{Remark}
  To prove Theorem \ref{thm:triangle}, we apply the duality theory in \cite{daskalakis2013mechanism,daskalakis2015multi} to our setting. We provide a brief description here and refer readers to \cite{daskalakis2013mechanism} and \cite{daskalakis2015multi} for details.

  Let $f(v)$ be the joint value distribution of $v=(v_1, v_2)$, and $V$ be the support of $f(v)$. Define measures $\mu_0$, $\mu_\partial$, $\mu_s$ as follows:
  \begin{itemize}
    \itemsep0em
  	\item $\mu_0$ has a single point mass at $v=0$, i.e., $\mu_0(V)=\mathcal{I}(\underline{v}\in A)$, where $\mathcal{I}(\cdot)$ is the indicator function, and $\underline{v}\in A$ is the smallest type in $V$.
  	\item $\mu_\partial$ is only distributed along the boundary of $V$, with a density $f(v)(v\cdot \eta(v))$, where $\eta(v)$ is the outer unit normal vector at $v$.
  	\item $\mu_s$ is distributed in $V$ with a density $\nabla f(v)\cdot v+(n+1)f(v)$, where $n$ is the number of items.
  \end{itemize}

  Let $\mu=\mu_0+\mu_\partial-\mu_s$. Define $\mu_+$ and $\mu_-$ to be two non-negative measures such that $\mu=\mu_+-\mu_-$. Let $V_+$ and $V_-$ be the support sets of $\mu_+$ and $\mu_-$. \cite{daskalakis2013mechanism,daskalakis2015multi} shows that designing an optimal mechanism for selling $n$ items to 1 buyer is equivalent to solving the following program:
  \begin{align}
  	\sup \quad&\int_{V}u\,\mathrm{d}\mu_+-\int_{V}u\,\mathrm{d}\mu_- \nonumber\\
  	\text{s.t.} \quad&u(v)-u(v')\le \|(v-v')_+\|_1, \forall v\in V_+,v'\in V_- \tag{P}\label{eq:primal}\\
  	&u\text{~is convex}, \;\; u(\underline{v})=0 \nonumber
  \end{align}
  where $u(v)$ is the utility of the buyer when his value is $v$, and $\|(v-v')_+\|_1=\sum_{i=1}^{n}\max (0,v_i-v'_i)$.

  Relax the above program by removing the convexity constraint and write the dual program of the relaxed program:
  \begin{align}
  	\inf\quad&\int_{V\times V}\|(v-v')_+\|_1\,\mathrm{d}\gamma\nonumber \\
  	\text{s.t.}\quad&\gamma \in \Gamma(\mu_+, \mu_-)\tag{D}\label{eq:dual}
  \end{align}
  where $\Gamma(\mu_+, \mu_-)$ is the set of non-negative measures $\gamma$ defined over $V\times V$ such that, for any $V'\subseteq V$, the following equations hold:
  \begin{gather*}
  	\int_{V'\times V}\,\mathrm{d}\gamma=\mu_+(V')\quad\text{and}\quad\int_{V\times V'}\,\mathrm{d}\gamma=\mu_-(V')
  \end{gather*}
  \begin{lemma}[\citet{daskalakis2013mechanism}]
  	\eqref{eq:dual} is a weak dual of \eqref{eq:primal}.
  \end{lemma}
  We omit the proof here but refer readers to \cite{daskalakis2013mechanism} and  \cite{daskalakis2015multi} for details. The dual program \eqref{eq:dual} has an optimal transport interpretation. We ``move'' the mass from $\mu_+$ to other points to form $\mu_-$ and the measure $\gamma$ corresponds to the amount of mass that goes from each point to another in $V$.

  Although \eqref{eq:dual} is only a weak dual of \eqref{eq:primal}, we can still use it to certify the optimality of a solution. We already give a menu in Theorem \ref{thm:triangle}. Therefore, the relaxed convexity constraint is automatically satisfied if the buyer always choose the best menu item.

  In our setting, $f(v)=\frac{2}{c}$, and we have that $V=T$,  $\underline{v}=(0,0)$, $\mu_\partial$ has a constant line density of $\frac{2}{\sqrt{1+c^2}}$ along the segment $\frac{v_1}{c}+v_2=1, 0\le v_2\le 1$, and $\mu_s$ has a constant density of $\frac{6}{c}$ over $T$.

  Let $R_i$ be the region of $T$ such that for any $v\in R_i$, choosing menu item $i$ maximizes the buyer's utility.

  It is straightforward to verify that the measures $\mu_+$ and $\mu_-$ are balanced inside each region, i.e., $\mu_+(R_i)=\mu_-(R_i), \forall i$. Therefore, the transport of mass only happens inside each region.

  We construct the transport in $R_1$ and $R_2$ as follows:
  \begin{itemize}
    \itemsep0em
  	\item $R_1$: $\mu_+$ is concentrated on a single point $0$. We move the mass at $0$ uniformly to all points in $R_1$;
  	\item $R_2$: $\mu_+$ is only distributed along the upper boundary of $R_2$. For each point $v$ at the upper boundary, we draw a vertical  line $l$ through it, and move the mass at $v$ uniformly to the points in $L\cap R_2$.
  \end{itemize}

  However, for $R_3$, $\mu_+$ is also only distributed along the upper boundary, but there is no easy transport as for $R_1$ and $R_2$. We provide the following Lemma \ref{lem:left_bottom_trans}.
  \begin{lemma}
  	\label{lem:left_bottom_trans}
  	For $R_3$, there exists a transport of mass, such that for any two points $v,v'$, if there is non-negative transport from $v$ to $v'$, then $v_i\ge v'_i, \forall i$.
  \end{lemma}

  The proof of Lemma \ref{lem:left_bottom_trans} is deferred to Appendix \ref{appen:left_bottom_trans}. With this lemma, we can simplify our proof of Theorem \ref{thm:triangle}, and do not need to construct the measure $\gamma$ explicitly.
  \begin{proof}[Proof of Theorem \ref{thm:triangle}]
  	Point $D$ in Figure \ref{fig:R_3} has coordinates $(x_D,y_D)$, where $x_D=\frac{2}{3}c-\frac{1}{3}\sqrt{c(c-1)}-\frac{1}{3}\sqrt{\frac{c}{c-1}}$ and $y_D=\frac{1}{3}\sqrt{\frac{c}{c-1}}$. Therefore,
  	\begin{gather*}
  		\mathrm{{\bf Pr}}\{\text{The buyer chooses menu item 2}\}=f(v)\cdot S(YCDI)=\frac{2}{c}\cdot \frac{1}{3}x_D\\
  		\mathrm{{\bf Pr}}\{\text{The buyer chooses menu item 3}\}=f(v)\cdot S(CDEX)=\frac{2}{c}\left[\frac{c}{2}\left(\frac{1}{3}+y_D \right)^2-\frac{1}{2}y_D^2 \right]
  	\end{gather*}
  	Thus the revenue of the menu provided in Theorem \ref{thm:triangle} is:
  	\begin{align*}
  		\Rev=~&\frac{2}{3}\cdot \mathrm{{\bf Pr}}\{\text{The buyer chooses menu item 2}\}\\
  		&+\left(\frac{2}{3}c-\frac{1}{3}\sqrt{c(c-1)}\right)\cdot \mathrm{{\bf Pr}}\{\text{The buyer chooses menu item 3}\}\\
  		=~&\frac{2}{27}\left[4+c+\sqrt{c(c-1)} \right]
  	\end{align*}

  	New we compute the objective of the dual program \eqref{eq:dual}. And to prove the optimality of the menu, it suffices to show that the objective of \eqref{eq:dual} is equal to $\Rev$.

  	Note that in our construction of the transport in $R_1$ and $R_2$, we only allow transport inside each region. In $R_1$, we transport mass from point 0 to other points. So it does not contribute to the objective of \eqref{eq:dual}, and we can just ignore $R_1$. In $R_2$, the mass is always moved vertically down. Therefore, for any $v,v'$, such that there is positive mass transport from $v$ to $v'$, we have $v_i\ge v'_i,\forall i$ and  $\|(v-v')_+\|_1=\sum_{i}\max(0, v_i-v'_i)=\sum_{i}(v_i-v'_i)=\sum_{i}(v_i-0)-\sum_{i}(v'_i-0)$. Therefore,
  	\begin{gather}
  		\int_{R_2\times R_2}\|(v-v')_+\|_1\,\mathrm{d}\gamma=\int_{R_2\times R_2}\|v-0\|_1\,\mathrm{d}\gamma-\int_{R_2\times R_2}\|v'-0\|_1\,\mathrm{d}\gamma\tag{4}
  		\label{eq:potential}
  	\end{gather}
  	For the first term, we have:
  	\begin{align*}
  		\int_{R_2\times R_2}\|v-0\|_1\,\mathrm{d}\gamma=\int_{R_2\times T}\|v-0\|_1\,\mathrm{d}\gamma=\sum_{j}\int_{\sigma_j\times T}\|v-0\|_1\,\mathrm{d}\gamma
  	\end{align*}
  	where the first equation is due to the fact that our transport is inside each region, and $\{\sigma_j\}$ is a partition of the region $R_2$. When the maximum area of $\sigma_j$ approaches 0, we get:
  	\begin{align*}
  		&\int_{R_2\times R_2}\|v-0\|_1\,\mathrm{d}\gamma=\int_{R_2}\|v-0\|_1\,\mathrm{d}\mu_+\\
  		=~&\int_{0}^{x_D}\left(v_1+1-\frac{v_1}{c} \right)\frac{2}{\sqrt{1+c^2}}\frac{\sqrt{1+c^2}}{c}\,\mathrm{d}v_1
  		=\frac{1}{9}\left(8-6\sqrt{\frac{c}{c-1}}+5c-4\sqrt{c(c-1)} \right)
  	\end{align*}

  	Similarly, the second term of Equation \eqref{eq:potential} is:
  	\begin{align*}
  		\int_{R_2\times R_2}\|v'-0\|_1\,\mathrm{d}\gamma&=\int_{R_2}\|v'-0\|_1\,\mathrm{d}\mu_-=\frac{\left(2\sqrt{c-1}-\sqrt{c}\right)\left(3+2c-\sqrt{c(c-1)} \right)}{9\sqrt{c-1}}
  	\end{align*}

  	For $R_3$, according to Lemma \ref{lem:left_bottom_trans}, it is also true that when there is positive mass transport from $v$ to $v'$, we always have $v_i\ge v'_i,\forall i$. Therefore,
  	\begin{gather*}
  		\int_{R_3\times R_3}\|(v-v')_+\|_1\,\mathrm{d}\gamma=\int_{R_3\times R_3}\|v-0\|_1\,\mathrm{d}\gamma-\int_{R_3\times R_3}\|v'-0\|_1\,\mathrm{d}\gamma
  	\end{gather*}
  	For the first term,
  	\begin{align*}
  		\int_{R_3\times R_3}\|v-0\|_1\,\mathrm{d}\gamma&=\int_{x_D}^{c}\left(v_1+1-\frac{v_1}{c}\right)\frac{2}{c}\,\mathrm{v_1}=\frac{1}{9}\left(1+4c+\frac{4c\sqrt{c}}{c-1}+2\sqrt{\frac{c}{c-1}} \right)
  	\end{align*}
  	Similarly, for the second term,
  	\begin{align*}
  		\int_{R_3\times R_3}\|v'-0\|_1\,\mathrm{d}\gamma&=\int_{v\in R_3}\frac{6}{c}(v_1+v_2)\,\mathrm{d}v=\frac{1}{27}\left(1+5\sqrt{\frac{c}{c-1}}+10c+10c\sqrt{\frac{c}{c-1}} \right)
  	\end{align*}
  	Therefore, the objective of the dual program \eqref{eq:dual} is:
  	\begin{align*}
  		\int_{T\times T}\|(v-v')_+\|_1\mathrm{d}\gamma =\int_{R_2\times R_2}\|(v-v')_+\|_1\mathrm{d}\gamma+\int_{R_3\times R_3}\|(v-v')_+\|_1\mathrm{d}\gamma
  		=\frac{2}{27}\left[4+c+\sqrt{c(c-1)} \right] =\Rev
  	\end{align*}
  	The above equation shows that the dual objective is equal to the actual revenue, which certifies that the menu is optimal.
  \end{proof}

  When $c\le\frac{4}{3}$, the optimal mechanism only has two menu items.
  \begin{theorem}
  	\label{thm:triangle2}
  	For any $1\le c\le\frac{4}{3}$, suppose that the buyer's type is uniformly distributed among the set $T=\{(v_1,v_2)~|~\frac{v_1}{c}+v_2\le 1, v_1\ge  0, v_2\ge 0 \}$. Then the optimal menu contains the following two items:
      $(0,0),0$ and
      $(1,1),\sqrt{\frac{c}{3}}$.
  \end{theorem}
  One can prove Theorem \ref{thm:triangle2} with the same trick in Lemma \ref{lem:left_bottom_trans}.  We omit the proof of this theorem since it is easier compared to the other case described in Theorem \ref{thm:triangle}.

  \subsubsection{Optimal mechanisms under limited menu size constraints}
	In this section, we consider the optimal $3$-Menu Mechanisms for value distribution $U[0, 1]^2$.

  \begin{theorem}\label{thm:symm3menu}
    The optimal symmetric at-most-three-menu mechanism for two additive items
    with $\v \sim U[0, 1]^2$ is to sell the bundle of two items at price
    $\sqrt6 / 3$, yielding revenue $2\sqrt6 / 9 \approx 0.54433$.
  \end{theorem}
  We defer the proof to \autoref{sec:proofs}.

  \begin{theorem}\label{thm:3menu}
    The optimal at-most-three-menu mechanism for two additive items with $\v
    \sim U[0, 1]^2$ is to sell the first item at price $2 / 3$ or the bundle of
    two items at price $5 / 6$, yielding revenue $59 / 108 \approx 0.546296$.

    By symmetry, the mechanism could also be selling the second item at price
    $2 / 3$ or the bundle of two items at price $5 / 6$. In particular, these
    is no other at-most-three-menu mechanisms could generate as much revenue as
    they do.
  \end{theorem}

  We demonstrate the proof through the basic parametric method. Note that there
  must be a zero menu $Z = [(0, 0), 0]$, and hence we have two menus to
  determine. Suppose that the remaining two menus are $A = [(\alpha, \beta), p]$
  and $B = [(\gamma, \delta), q]$. We then solve the following problem:
  \begin{align}\label{eq:3menurevmax}\tag{$3\Menu$}
    \begin{aligned}
      \text{maximize}   &\quad \Rev(A, B, Z)  \\
      \text{subject to} &\quad \alpha, \beta, \gamma, \delta \in [0, 1],~
        p, q \geq 0.
    \end{aligned}
  \end{align}
  To establish the connection between the menus and the revenue, let $S_A$ be
  the set of values that menu $A$ is the most preferred:
  \begin{align*}
    S_A = \{(v_1, v_2) \in [0, 1]^2 | (v_1, v_2) \cdot (\alpha, \beta) - p
              \geq (v_1, v_2) \cdot (\gamma, \delta) - q ~\wedge~
              (v_1, v_2) \cdot (\alpha, \beta) - p \geq 0\}.
  \end{align*}
  Similarly, we define $S_B$ and $S_Z$ be the set of values where menu $B$ and
  menu $Z$ are the most preferred, respectively:
  \begin{gather*}
    S_B = \{(v_1, v_2) \in [0, 1]^2 | (v_1, v_2) \cdot (\gamma, \delta) - q
              \geq (v_1, v_2) \cdot (\alpha, \beta) - p ~\wedge~
              (v_1, v_2) \cdot (\gamma, \delta) - q \geq 0\},  \\
    S_Z = \{(v_1, v_2) \in [0, 1]^2 |
              0 \geq (v_1, v_2) \cdot (\alpha, \beta) - p ~\wedge~
              0 \geq (v_1, v_2) \cdot (\gamma, \delta) - q\}.
  \end{gather*}
  For any measurable set $S \subseteq [0, 1]^2$, let $|S| = \Pr[(v_1, v_2) \in
  S]$ be the probabilistic measure of $S$. Then the revenue of the mechanism
  with menus $A$, $B$, and $Z$ is
  \begin{align}\label{eq:3menurev}\tag{$3\Menu\Rev$}
    \Rev(A, B, Z) = |S_A| \cdot p + |S_B| \cdot q.
  \end{align}

  With the above formulation, there are two major challenges to solve the
  program \eqref{eq:3menurevmax}:
  \begin{itemize}
    \itemsep0em
    \item There are too many possible cases with different formulas of $|S_A|$
          and $|S_B|$, hence the formula of $\Rev(A, B, Z)$. In particular,
          there are $4$ possible intersection patterns between the boundary of
          the square $[0, 1]^2$ and the intersection of each two of the menus
          ($S_A \cap S_B$, $S_B \cap S_Z$, $S_Z \cap S_Z$). Hence roughly $4^3 =
          64$ different cases.
    \item Even within each specific case, the revenue $\Rev$ is still a
          high-order function with $6$ variables. In general, there is no
          guarantee for closed-form solutions.
  \end{itemize}

  To overcome these two challenges, the following two lemmas are critical to
  reducing both the number of different cases and free variables:
  \begin{Lemma}\label{lem:3menubundle}
    Without loss of generality, we can assume that the optimal at-most-three-menu
    mechanism includes bundling, $(1, 1)$, as one of its menu.
  \end{Lemma}
  \begin{proof}[Proof of \autoref{lem:3menubundle}]
    Without loss of generality, suppose that $p \geq q$, and then there must be
    an optimal mechanism with $\alpha = \beta = 1$. Because by replacing menu
    $A$ with menu $A' = [(1, 1), p]$, the set of values where $A'$ dominating
    $B$ and $Z$, $S'_{A'}$ will be a superset of $S_A$, and similarly, $S'_{Z}$
    will be a subset of $S_Z$, i.e., $S'_{A'} \supseteq S_A ~\text{and}~ S'_{Z}
    \subseteq S_Z$. Therefore,
    \begin{gather*}
      \Rev' = |S'_{A'}| \cdot p + |S'_B| \cdot q
        = |S'_{A'}| \cdot (p - q) + (1 - |S'_Z|) \cdot q
        \geq |S_A| \cdot (p - q) + (1 - |S_Z|) \cdot q = \Rev.
    \end{gather*}
  \end{proof}
  \begin{Lemma}[{\citet[Proposition 2]{pavlov2011property}}]\label{lem:pavlov}
    For $\v \sim U[0,1]^2$, consider a mechanism with a menu $(\gamma, \delta)$
    such that $\gamma,\delta \neq 1$ and $(\gamma, \delta) \neq (0, 0)$, then
    by replacing the menu with $(\gamma', \delta')$ (the price of the menu may
    also be different), the revenue of the new mechanism is no less than the
    original mechanism, where $\gamma' = 1$ or $\delta' = 1$ or $(\gamma',
    \delta') = (0, 0)$.
  \end{Lemma}
  \begin{figure}
    \subfigure[Case 1]
      {\includegraphics[width=0.28\textwidth]{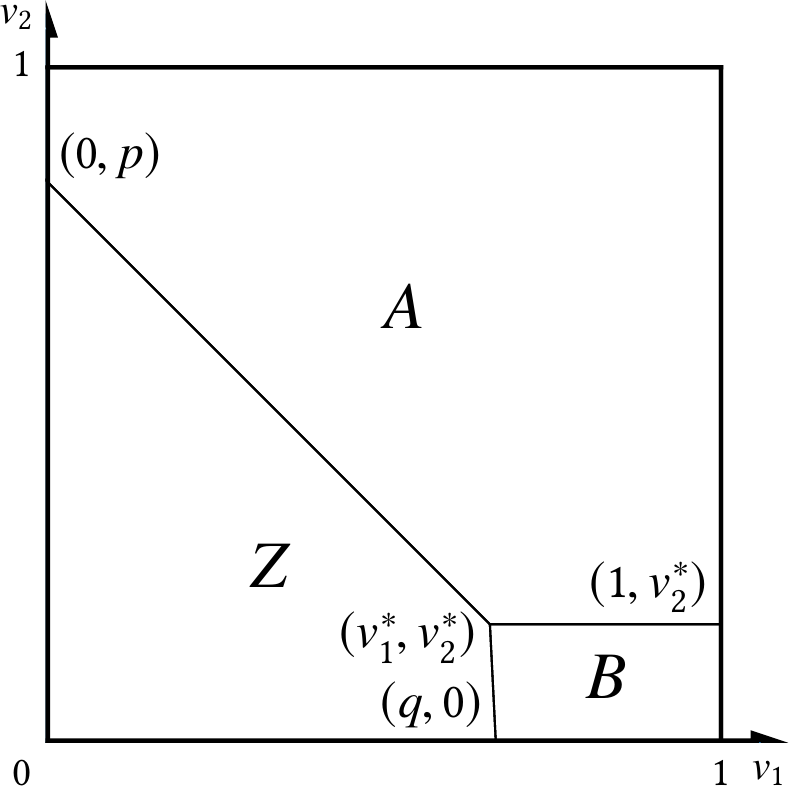}
      \label{sfig:case1}} %
    \hfill %
    \subfigure[Case 2]
      {\includegraphics[width=0.28\textwidth]{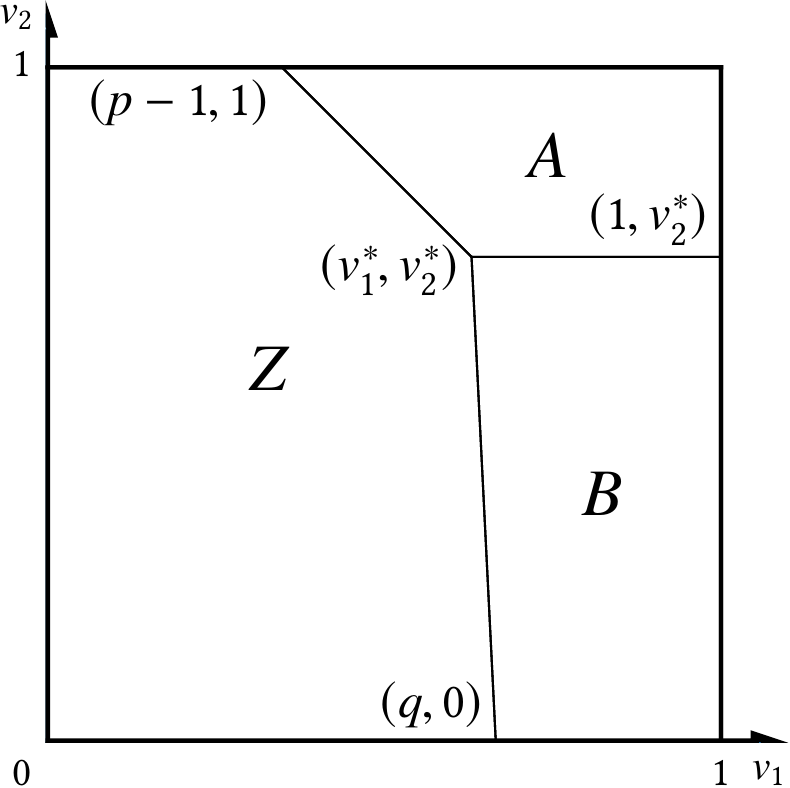}
      \label{sfig:case2}} %
    \hfill %
    \subfigure[Case 3]
      {\includegraphics[width=0.28\textwidth]{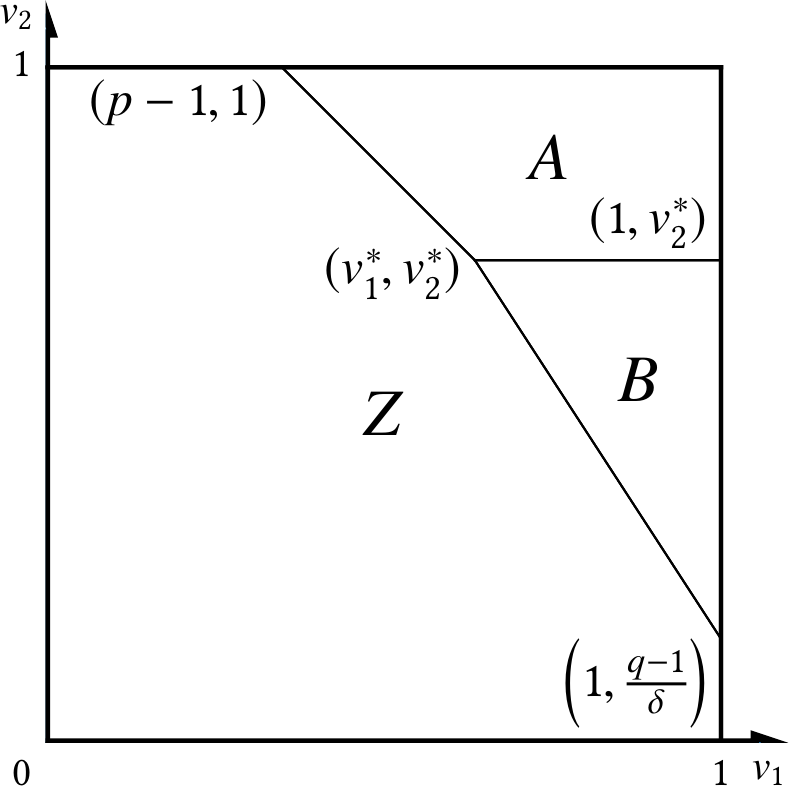}
      \label{sfig:case3}}
    \caption{Three possible cases for the proof of \autoref{thm:3menu}.}
    \label{fig:3mcase}
  \end{figure}
  \begin{proof}[Proof of \autoref{thm:3menu}]
    By \autoref{lem:3menubundle}, we can fix $\alpha = 1$ and $\beta = 1$.
    Moreover, without loss of generality, we could focus on the cases with $p >
    q$. Otherwise, the menu $B$ will be dominated by menu $A$ and menu $Z$,
    i.e., $S_B = \emptyset$, hence reduced to a two-menu mechanism, where the
    optimal revenue is at most $2\sqrt6 / 9$.

    Similarly, by \autoref{lem:pavlov}, we can fix one of $\gamma$ and $\delta$
    to be $1$, without loss of generality, $\gamma = 1$. Note that in the case
    with $(\gamma, \delta) = (0, 0)$, menu $B$ will be dominated by menu $Z$,
    hence reduced to a two-menu mechanism again.

    Therefore, we remain to solve \eqref{eq:3menurevmax} with additional
    constraints: $\alpha = \beta = \gamma = 1$ and $p > q$.

    Now consider the values $\v = (v_1, v_2)$ in $S_A \cap S_B$, which must
    satisfy:
    \begin{align*}
      S_A \cap S_B: (v_1, v_2) \cdot (1, 1) - p
                      = (v_1, v_2) \cdot (1, \delta) - q.
    \end{align*}
    Similarly,
      $S_A \cap S_Z: (v_1, v_2) \cdot (1, 1) = p, 
      S_B \cap S_Z: (v_1, v_2) \cdot (1, \delta) = q$,
    and hence
      $S_A \cap S_B \cap S_Z:
        v^*_1 = \frac{q - \delta p}{1 - \delta},
        v^*_2 = \frac{p - q}{1 - \delta}$.
    Note that if $S_A$ or $S_B$ is empty, there would be only two menus and the
    revenue cannot be more than $2\sqrt{6} / 9$. Otherwise:
    \begin{itemize}
      \itemsep0em
      \item For $S_A$ not being empty, we must have $v^*_2 < 1$, hence:
            \begin{align}\label{eq:a}\tag{$\textsc{NonEmptyA}$}
              \frac{p - q}{1 - \delta} < 1;
            \end{align}
      \item For $S_B$ not being empty, we must have $v^*_1 < 1$, hence:
            \begin{align}\label{eq:b}\tag{$\textsc{NonEmptyB}$}
              \frac{q - \delta p}{1 - \delta} < 1.
            \end{align}
    \end{itemize}
    Based on the constraints \eqref{eq:a} and \eqref{eq:b}, there are three
    possible cases (see \autoref{fig:3mcase}). The solutions under these cases
    are summarized by the following lemmas.
    \begin{Lemma}[Case 1]\label{lem:case1}
      Conditional on $p \leq 1$, the optimal mechanism consists of asymmetric
      three menus $A: [(1, 1), 5 / 6]$, $B: [(1, 0), 2 / 3]$, $Z: [(0, 0), 0]$,
      and yields revenue $59 / 108$.
    \end{Lemma}
    \begin{Lemma}[Case 2]\label{lem:case2}
      Conditional on $p \geq 1 > q$, the optimal mechanism yields revenue $14 /
      27$.
    \end{Lemma}
    \begin{Lemma}[Case 3]\label{lem:case3}
      Conditional on $p > q > 1$, the revenue of the mechanism is not more than
      $1 / 2$.
    \end{Lemma}
    In summary, the optimal mechanism with at most $3$ menus is to sell the
    first item at price $2 / 3$ or the bundle of two items at price $5 / 6$,
    yielding revenue $59 / 108$.
  \end{proof}

\section{Performance}\label{sec:perf}

  \paragraph{Setup}
  As our method is very efficient, we were able to perform our experiments on a
  laptop ($13$-inch MacBook Pro, with 2.5 GHz Intel Core i7 CPU, 16 GB RAM)
  using TensorFlow. To solve the problems with continuous value distributions in
  finite neural networks, we simply discretize the value space. In particular,
  the discretization is parameterized by $N$, which is the number of the
  intervals (with length $1 / N$) in unit length. In other words, there are
  $N^2$ squares of size $1 / N$ by $1 / N$ in any unit square. By default, we
  set $N = 100$.

  \subsection{Efficiency and Accuracy: Compared with Linear Programs}

  \begin{figure}[!htb]
    \subfigure[Ours with linear program.]
      {\includegraphics[width=0.42\textwidth]{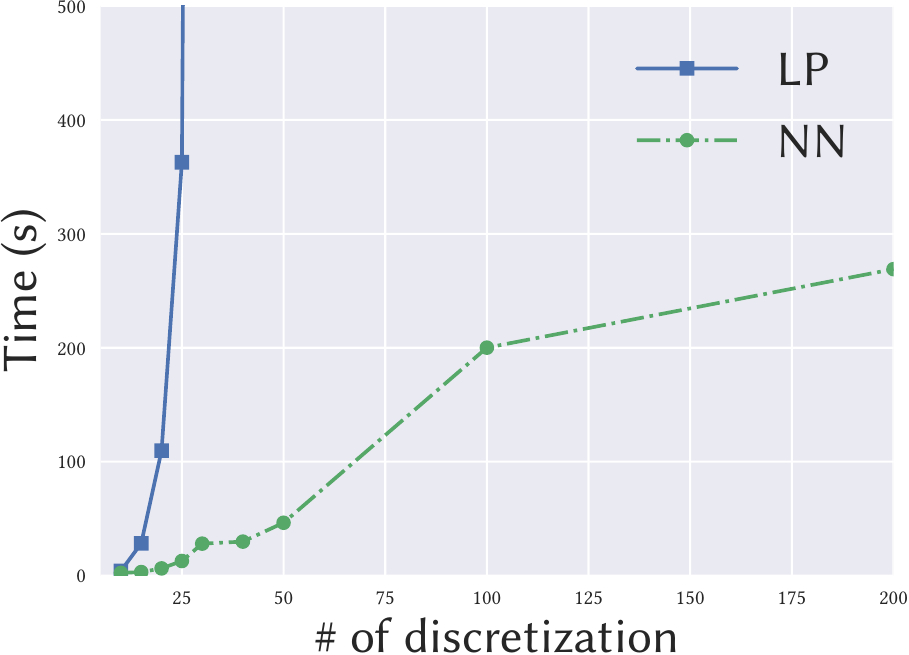}
      \label{sfig:time-lp-nn}} %
    \hfill %
    \subfigure[Average per iter.]
      {\includegraphics[width=0.42\textwidth]{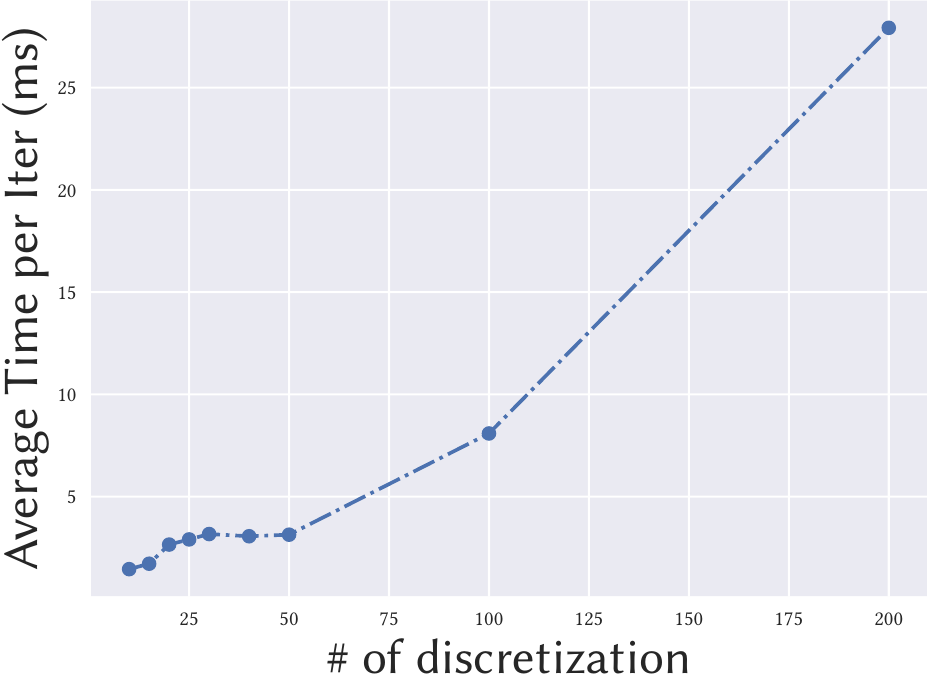}
      \label{sfig:time-nn-avg}} %
    \hfill 
    \subfigure[$\Rev / \Opt\Rev$ vs \# of iters.]
      {\includegraphics[width=0.42\textwidth]{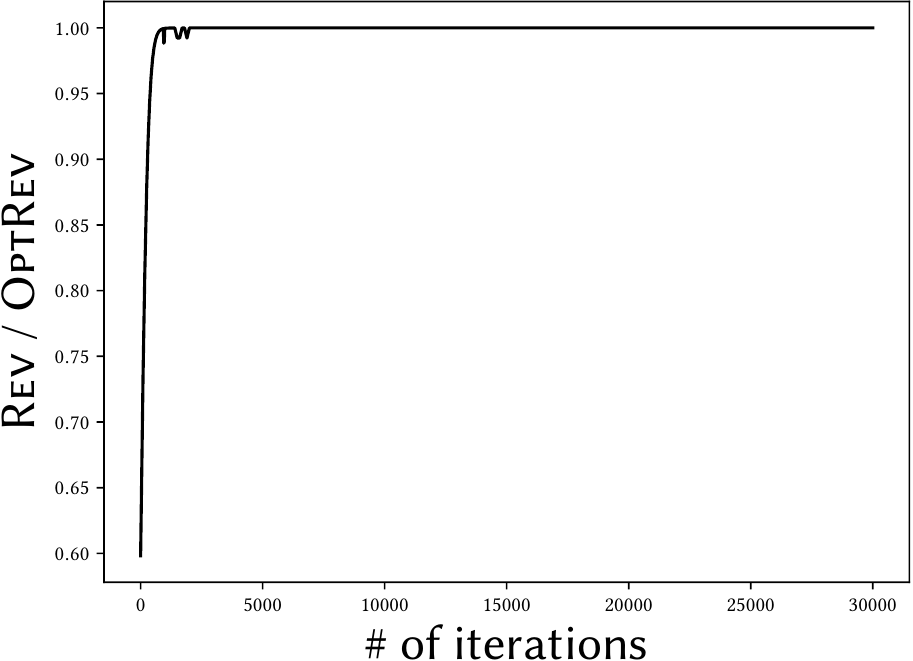}
      \label{sfig:revconv}} %
    \hfill %
    \subfigure[$1 - \frac{\Rev}{\Opt\Rev}$ vs \# of iters.]
      {\includegraphics[width=0.42\textwidth]{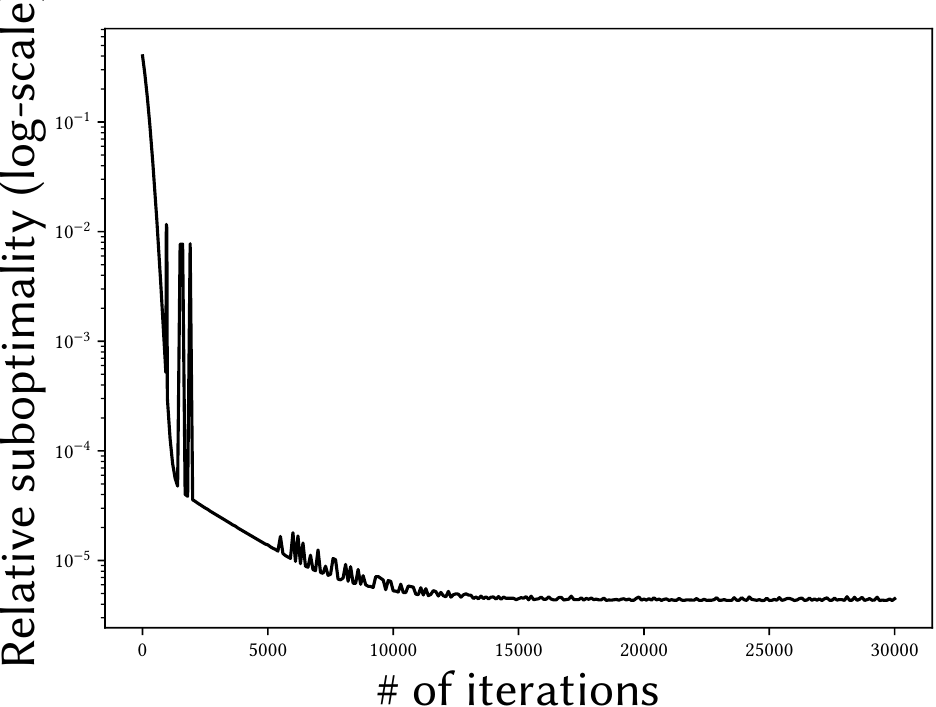}
      \label{sfig:errconv}}
    \caption{Running time and converge speed.}
    \label{fig:eacs}
  \end{figure}

  We compare the running time of our method and the straightforward linear
  program approach for the $U[0, 1]^2$ setting. In the linear program, the
  variables are the allocation $x_1, x_2$ and payment $p$ of the values on each
  discretized grid (hence $O(N^2)$ variables) and the constraints are the
  \ref{eq:ic} and \ref{eq:ir} constraints (hence $O(N^4)$ constraints). We use
  the basic PuLP package in Python to solve the linear programs. In
  \autoref{sfig:time-lp-nn}, we compared the execution time of solving the
  linear programs with specific $N$'s ($N = 10, 15, 20, 25, 30$) and the
  execution time of training our neural network to (i) achieve a mechanism with
  at least the same level of acurracy as the one given by the linear program
  (for $N \leq 30$), and (ii) converge (for $N = 40, 50, 200$). Note that the
  running time of the linear program approach grows very rapidly: for $N = 30$,
  it takes $51$ mins and we are not able to apply it to $N \geq 40$. In
  contrast, the training time of our neural network grows much slower (less than
  $5$ mins for $N = 200$, i.e., buyer distribution support of size $40000$).

  One key advantage of our approach over the linear program is that our problem
  size grows linearly in terms of the support size of the buyer's distribution
  (i.e., $O(N^2)$), while the size of the linear program grows quadratically in
  terms of the support size (i.e., $O(N^4)$). In \autoref{sfig:time-nn-avg}, we
  also plot the average training time for each iteration, which is in $1 \sim
  30$ milliseconds.

  \autoref{sfig:revconv} and \autoref{sfig:errconv} illustrates that our method
  converges to the optimal very fast. The relative error also drops very fast
  even in the log-scale plot. In particular, $\Rev$ is evaluated on the original
  continuous distribution $U[0, 1]^2$. Hence the gap between $\Rev$ and
  $\Opt\Rev$ cannot drop to zero as we discretized the value distribution.

  \paragraph{Conclusion}
  So far, we have shown that our approach is much more efficient than the linear
  program appraoch and hence much stronger scalability as well. To completement
  the time efficiency, we also show in \autoref{sec:compacc} that our method
  also dominates the linear program approach in terms of accuracy.


\bibliographystyle{plainnat}
\bibliography{gam}

\appendix
\section{Missing Proofs}\label{sec:proofs}
	\subsection{Proof of Lemma \ref{lem:left_bottom_trans}}\label{appen:left_bottom_trans}
	\begin{proof}[Proof of Lemma \ref{lem:left_bottom_trans}]
		Denote the upper boundary of $R_3$ by $B$. For each $v\in B$, define
		\begin{gather*}
			R_L=\{v'\in R_3~|~v'_1\le v_1 \}\quad\text{and}\quad R_U=\{v'\in R_3~|~v'_2\ge v_2 \}.
		\end{gather*}
		For any line $l_v$ through $v$ with a non-negative slope (or infinity), denote the part of $R_3$ that is above the line by $R_v$. It is easy to verify that $\mu_+(R_U)\ge \mu_-(R_U)$ and $\mu_+(R_L)\le \mu_-(R_L)$. Thus there exists a line $l_v^*$ such that the corresponding $R_v^*$ satisfies $\mu_+(R^*_v)=\mu_-(R_v^*)$.
		
		Now we show that for any two $v$ and $v'$, the intersection point of $l^*_v$ and $l^*_{v'}$ is not inside $R_3$. In Figure \ref{fig:R_3}, the three regions $R_1$, $R_2$, $R_3$ are the quadrangles $OIDE$, $YCDI$ and $CDEX$, respectively. Let points $A$, $B$ correspond to the value profiles $v$ and $v'$. Assume, on the contrary, that the intersection point of $l^*_v$ (line $AA'$) and $l^*_{v'}$ (line $BB'$) is inside $R_3$. Then we have:
		\begin{gather}
			\mu_+(ACDA')=\mu_-(ACDA')\quad\text{and}\quad\mu_+(BCDB')=\mu_-(BCDB')\tag{1}\label{eq:eq}.
		\end{gather}
		
		\begin{figure}[h]
			\centering
			\includegraphics[width=0.8\linewidth]{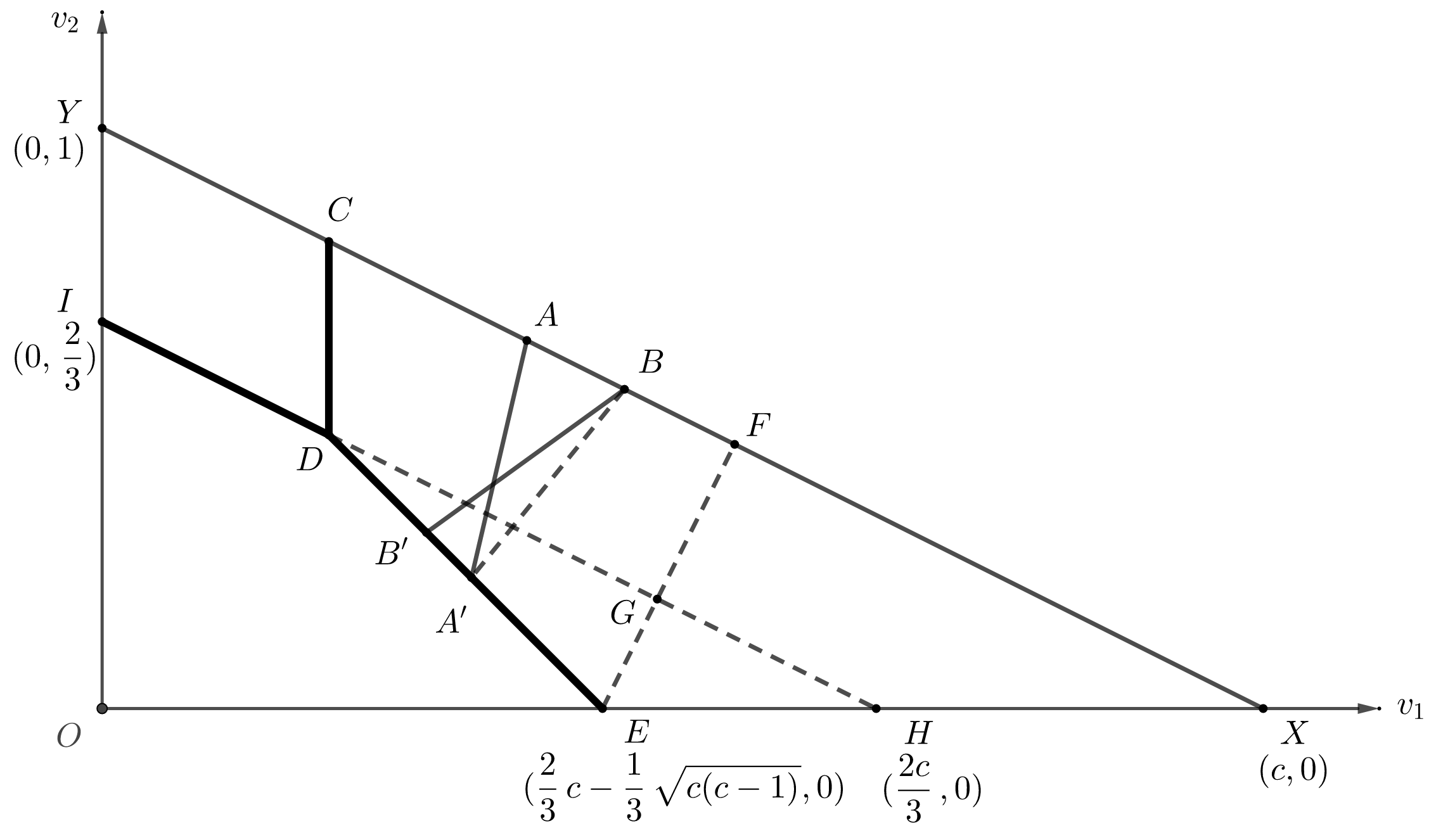}
			\caption{The intersection point of $l^*_v$ and $l^*_{v'}$}
			\label{fig:R_3}
		\end{figure}
		
		Note that $\mu_+$ is only distributed along the line $CX$ inside $R_3$. Thus
		\begin{gather*}
			\mu_+(BCDA')=\mu_+(BCDB')=\mu_-(BCDB')\tag{2}\label{eq:mu_pos}.
		\end{gather*}
		However, $\mu_-$ has a positive density inside $R_3$. Therefore, we have
		\begin{gather*}
			\mu_-(BCDA')>\mu_-(BCDB')\tag{3}\label{eq:mu_neg}.
		\end{gather*}
		Combining equations \eqref{eq:eq}, \eqref{eq:mu_pos} and \eqref{eq:mu_neg}, we obtain:
		\begin{gather*}
			\mu_+(BAA')<\mu_-(BAA')\tag{*}\label{eq:star}.
		\end{gather*}
		
		Since $\mu_+$ is uniformly distributed along the line $CX$ with density , we have that $\mu_+(BAA')=\frac{2}{\sqrt{1+c^2}}\cdot l(AB)$, where $l(\cdot)$ denotes the length of a segment. Similarly, $\mu_-(BAA')=\frac{6}{c}\cdot S(BAA')$, where $S(BAA')$ is the area of triangle $BAA'$. Let $h$ be the altitude of the triangle $BAA'$ with respect to the base $AB$. So
		\begin{align*}
			S(BAA')&=\frac{1}{2}l(AB)\cdot h\\
			&\le \frac{1}{2}l(AB)\cdot l(EF)\\
			&=\frac{1}{2}l(AB)\cdot \frac{c-\left(\frac{2}{3}c-\frac{1}{3}\sqrt{c(c-1)}\right)}{\sqrt{1+c^2}}\\
			&=l(AB)\cdot \frac{c+\sqrt{c(c-1)}}{6\sqrt{1+c^2}}
		\end{align*}
		where line $EF$ is perpendicular to line $CX$.  Then
		\begin{align*}
			\mu_-(BAA')&=\frac{6}{c}\cdot S(BAA')\\
			&\le \frac{6}{c}\cdot l(AB)\cdot \frac{c+\sqrt{c(c-1)}}{6\sqrt{1+c^2}}\\
			&=l(AB)\cdot \frac{c+\sqrt{c(c-1)}}{c\sqrt{1+c^2}}\\
			&\le l(AB)\cdot \frac{2}{\sqrt{1+c^2}}\\
			&=\mu_+(BAA')
		\end{align*}
		which contradicts to Equation \eqref{eq:star}.
		
		Consider the set of lines $K=\{l^*_v~|~v\in B \}$. Since we have already shown that no two of these lines have an intersection point inside $R_3$, the line set $K$ actually cuts the region $R_3$ into ``slices''. And for any ``slice'' $s$, we have $\mu_+(s)=\mu_-(s)$. Therefore, for each point in $B$, we can find its corresponding ``slice'' and move its mass uniformly to all the points inside the ``slice''. And since $l^*_v$ always has a non-negative (or infinite) slope, we conclude that whenever there is a mass transport from $v$ to $v'$, we have $v_i\ge v'_i, \forall i$.
		
	\end{proof}

  \subsection{Proof of \autoref{thm:symm3menu}}

  \begin{proof}[Proof of \autoref{thm:symm3menu}]
    Since $(0, 0), 0$ must be one of the three symmetric menus, the other two
    must have the form of $(\alpha, \beta), p$ and $(\beta, \alpha), p$.

    Without loss of generality, assume that $\alpha \geq \beta$ and $\alpha > 0$
    (otherwise $\alpha = \beta = 0$, yielding $0$ revenue). Therefore, if
    $\alpha v_1 + \beta v_2 < p$ and $\beta v_1 + \alpha v_2 < p$, the buyer
    will choose the zero menu, $(0, 0), 0$.

    Consider the following two cases: (i) $p \leq \alpha$ and (ii) $p \geq
    \alpha$.

    If $p \leq \alpha$:
    \begin{align*}
        \Pr[\alpha v_1 + \beta v_2 < p \wedge \beta v_1 + \alpha v_2 < p]
      = 2\Pr[\alpha v_1 + \beta v_2 < p \wedge v_1 \geq v_2]
      = \frac{p}{\alpha} \cdot \frac{p}{\alpha + \beta}.
    \end{align*}
    Then the revenue is
    \begin{align*}
      \Rev = p \cdot \Pr[\text{the buyer didn't choose the zero menu}]
      = \left(1 - \frac{p^2}{\alpha(\alpha + \beta)}\right) \cdot p.
    \end{align*}
    Hence
    \begin{align*}
      \Rev = \left(1 - \frac{p^2}{\alpha^2}
                \cdot \frac{1}{1 + \beta / \alpha}\right) \cdot p
        \leq (1 - p^2 / 2) \cdot p
        = \sqrt{(1 - p^2 / 2) \cdot (1 - p^2 / 2) \cdot p^2}
        \leq 2\sqrt{6} / 9,
    \end{align*}
    where the first inequality is reached if and only if $\alpha = \beta = 1$
    and the second inequality is reached if and only if $1 - p^2 / 2 = p^2 \iff
    p = \sqrt{2 / 3}$.

    If $p \geq \alpha$:
    \begin{align*}
      \Pr[\alpha v_1 + \beta v_2 < p \wedge \beta v_1 + \alpha v_2 < p]
    = 2\Pr[\alpha v_1 + \beta v_2 < p \wedge v_1 \geq v_2]
    = \frac{p}{\alpha} \cdot \frac{p}{\alpha + \beta}
      + \left(\frac{p}{\alpha} - 1\right)^2.
    \end{align*}
    Hence
    \begin{align*}
      \Rev = \left(1 - \frac{p^2}{\alpha^2}
                \cdot \frac{1}{1 + \beta / \alpha}
                + \left(\frac{p}{\alpha} - 1\right)^2\right) \cdot p
           \leq \left(1 - \frac{p^2}{\alpha^2}
                     \cdot \frac12
                     + \left(\frac{p}{\alpha} - 1\right)^2\right) \cdot p
           = \frac{p}2\left(\frac{p^2}{\alpha^2} - 4 \cdot \frac{p}{\alpha} + 4
              \right)
    \end{align*}

    Let $x = p / \alpha \geq 1$, the right-hand-side becomes
    \begin{align*}
      \frac{\alpha}2 (x^3 - 4x^2 + 4x) \leq \frac12 (x^3 - 4x^2 + 4x).
    \end{align*}
    Then consider the first order derivative of $x^3 - 4x^2 + 4x$:
    \begin{align*}
      (x^3 - 4x^2 + 4x)' = 3x^2 - 8x + 4 = (3x - 2)(x - 2),
    \end{align*}
    the local maximum is reached at $x = 2 / 3$. Note that in this case, $x = p
    / \alpha \geq 1$. Hence the maximum revenue contional on $p \geq \alpha$ is
    reached when $p = \alpha = \beta = 1$, where $\Rev = 1/2 < 2\sqrt{6} / 9$.
  \end{proof}

  \subsection{Proofs for \autoref{thm:3menu}}

    \subsubsection{Proof of \autoref{lem:case1}}

    \begin{proof}[Proof of \autoref{lem:case1}]
      When $p \leq 1$, consider:
      \begin{gather*}
        S_B \cap S_Z \cap \{v_2 = 0\}: v_1 = q, v_2 = 0  \\
        S_A \cap S_B \cap \{v_1 = 1\}: v_1 = 1,
          v_2 = \frac{p - q}{1 - \delta}.
      \end{gather*}

      Note that $q \leq p \leq 1$:
      \begin{align*}
        |S_B| &~= \frac12 \cdot \left((1 - q) \cdot v^*_2
                    + \frac{p - q}{1 - \delta} \cdot (1 - v^*_1)\right)  \\
        |S_Z| &~= \frac12 \cdot \left(p^2 - (p - q) \cdot v^*_2\right)  \\
        |S_A| &~= 1 - |S_B| - |S_Z|
      \end{align*}

      Then the revenue is
      \begin{align*}
        \Rev ~&= (1 - |S_Z|) \cdot p - |S_B| \cdot (p - q)  \\
          &= \frac12\left(2p - p^3
                + (p - q)p \cdot v^*_2
                - (1 - q)(p - q) \cdot v^*_2
                - \frac{(p - q)^2}{1 - \delta} \cdot (1 - v^*_1)\right)  \\
          &= \frac12\left(2p - p^3
              - \frac{(p - q)^3}{(1 - \delta)^2}
              + \frac{(p - q)^2(2p + q - 2)}{1 - \delta}\right)
              \\
          &= \frac12\left(2p - p^3 + (p - q)^3 \cdot \left(-\left(\frac1{1 - \delta}
                - \frac{2p + q - 2}{2(p - q)}\right)^2 + \left(\frac{2p + q - 2}{2(p - q)}\right)^2
              \right)\right)  \\
          &\leq \frac12\left(2p - p^3 + (p - q)(p + q / 2 - 1)^2\right),
      \end{align*}
      where the upper bound is reached if and only if: (i) $p = q$, or (ii) $1 -
      \delta = 2(p - q) / (2p + q - 2)$. Remember that we have shown that $p \neq
      q$, hence we must have $1 - \delta = 2(p - q) / (2p + q - 2)$ and
      %
      %
      \begin{align*}
        2\Rev ~&= 2p - p^3 + (p - q)(p + q / 2 - 1)^2
          = 2p - p^3 + (p - q) \cdot (p + q / 2 - 1) \cdot (p + q / 2 - 1)  \\
          &\leq 2p - p^3 + \left(\frac{(p - q) + (p + q / 2 - 1) + (p + q / 2 - 1)}3\right)^3
          = 2p - p^3 + (p - 2 / 3)^3,
      \end{align*}
      where the upper bound is reached if and only if $p - q = p + q / 2 - 1$ or
      equivalently $q = 2 / 3$.

      Substituting $q$ with $2 / 3$, we have
      \begin{align*}
        \Rev = -p^2 + 5/3p - 4 / 27 \leq -(p - 5 / 6)^2 + 25 / 36 - 4 / 27,
      \end{align*}
      and its local maximum is reached when $p = 5 / 6$, hence
      \begin{align*}
        \Rev = 59 / 108 \approx 0.546296 > 2\sqrt6 / 9 \approx 0.54433,
      \end{align*}
      and the menus are:
      \begin{align*}
        A: [(1, 1), 5 / 6] \quad B: [(1, 0), 2 / 3] \quad Z: [(0, 0), 0].
      \end{align*}
    \end{proof}

    \subsubsection{Proof of \autoref{lem:case2}}

    \begin{proof}[Proof of \autoref{lem:case2}]
      When $p > 1 \geq q$, consider:
      \begin{gather*}
        S_B \cap S_Z \cap \{v_2 = 0\}: v_1 = q, v_2 = 0  \\
        S_A \cap S_B \cap \{v_1 = 1\}: v_1 = 1,
          v_2 = \frac{p - q}{1 - \delta}.
      \end{gather*}

      Hence
      \begin{align*}
        |S_B| &~= \frac12 \cdot \left((1 - q) \cdot v^*_2
                    + \frac{p - q}{1 - \delta} \cdot (1 - v^*_1)\right)  \\
        |S_Z| &~= 1 - |S_A| - |S_B|  \\
        |S_A| &~= \frac12 \cdot \left((2 - p)^2 - \left(\frac{p - q}{1 - \delta}
                  - (p - 1)\right) \cdot (1 - v^*_1)\right).
      \end{align*}
      Then the revenue is
      \begin{align*}
        \Rev ~&= |S_A| \cdot p + |S_B| \cdot q
         = \frac12\left(3p - 2p^2 + (p - q)^2 \cdot \left(-\frac{p - q}{(1 - \delta)^2}
            + \frac{2p + q - 2}{1 - \delta}\right)\right)  \\
        &= \frac12\left(3p - 2p^2 + (p - q)^3 \cdot \left(-\left(\frac1{1 - \delta}
              - \frac{2p + q - 2}{2(p - q)}\right)^2
              + \left(\frac{2p + q - 2}{2(p - q)}\right)^2\right)\right)  \\
        &\leq \frac12\left(3p - 2p^2 + (p - q)(p + q / 2 - 1)^2\right)
         \leq 3p / 2 - p^2 + (p - 2/3)^3 / 2  \\
        &= \frac1{54}\left(27 p^3 - 108 p^2 + 117 p - 8\right),
      \end{align*}
      where the two inequalities are reached if and only if $1 - \delta = 2(p - q)
      / (2p + q - 2)$ and $q = 2 / 3$.

      Note that we have to ensure $v^*_2 \leq 1 \iff (p - q) / (1 - \delta) \leq
      1$, in other words,
      \begin{align*}
        \frac{p - 2 / 3}{(p - 2 / 3) / (p + 1 / 3 - 1)} \leq 1
          \iff p \leq 5 / 3.
      \end{align*}

      Now consider the maximum of $(27 p^3 - 108 p^2 + 117 p - 8) / 54$ with $p
      \in [1, 5 / 3]$, by the first order condition, the local maximum and minimum
      are reached at $p = (4 - \sqrt3) / 3 \approx 0.75598 < 1$ and $p = (4 +
      \sqrt3) / 3 \approx 1.91068 > 5 / 3$, respectively. Therefore, in this case,
      the revenue is decreasing in $p$ and hence the maximum revenue is reached at
      $p = 1$: $\Rev(p = 1) = 14 / 27 < 59 / 108$.
    \end{proof}

    \subsubsection{Proof of \autoref{lem:case3}}

    \begin{proof}[Proof of \autoref{lem:case3}]
      When $p > q > 1$, consider:
      \begin{gather*}
        S_B \cap S_Z \cap \{v_1 = 1\}: v_1 = 1, v_2 = \frac{q - 1}{\delta}  \\
        S_A \cap S_B \cap \{v_1 = 1\}: v_1 = 1, v_2 = \frac{p - q}{1 - \delta}.
      \end{gather*}

      Hence
      \begin{align*}
        |S_B| &~= \frac12 \cdot \left(\frac{p - q}{1 - \delta}
                    - \frac{q - 1}{\delta}\right) \cdot (1 - v^*_1)  \\
        |S_Z| &~= 1 - |S_A| - |S_B|  \\
        |S_A| &~= \frac12 \cdot \left((2 - p)^2 - \left(\frac{p - q}{1 - \delta}
                  - (p - 1)\right) \cdot (1 - v^*_1)\right).
      \end{align*}

      Note that $(q \delta - p) / (1 - \delta) < 1$ by \eqref{eq:b}, hence
      \begin{align*}
        \frac{p - q}{1 - \delta} = p - \frac{q \delta - p}{1 - \delta} > p - 1.
      \end{align*}

      Then the revenue is
      \begin{align*}
        \Rev ~&= |S_A| \cdot p + |S_B| \cdot q  \\
         &= \frac12\left((2 - p)^2 \cdot p + \left(-\left(\frac{p - q}{1 - \delta}
                   - (p - 1)\right) \cdot p + \left(\frac{p - q}{1 - \delta}
                  - \frac{q - 1}{\delta}\right) \cdot q\right) \cdot (1 - v^*_1)
           \right)  \\
         &\leq \frac12\left((2 - p)^2 \cdot p + \left(-\left(\frac{p - q}{1 - \delta}
                   - (p - 1)\right) \cdot q + \left(\frac{p - q}{1 - \delta}
                  - \frac{q - 1}{\delta}\right) \cdot q\right) \cdot (1 - v^*_1)
           \right)  \\
         &= \frac12\left((2 - p)^2 \cdot p
              + \left(p - 1 - \frac{q - 1}{\delta}\right)
                      \cdot \left(1 - \frac{q \delta - p}{1 - \delta}\right)
                      \cdot q\right).
      \end{align*}

      In the meanwhile, note that by \eqref{eq:a}, $v^*_2 = (p - q) / (1 - \delta)
      < 1$, hence $1 - v^*_1 = 1 - (p - v^*_2) < 2 - p$. Therefore, we have
      \begin{align*}
        \Rev ~&= \frac12\left((2 - p)^2 \cdot p
             + \left(p - 1 - \frac{q - 1}{\delta}\right)
                     \cdot \left(1 - \frac{q \delta - p}{1 - \delta}\right)
                     \cdot q\right)  \\
          &\leq \frac12\left((2 - p)^2 \cdot p
               + \left(p - 1 - \frac{q - 1}{1}\right)
                       \cdot (2 - p) \cdot q\right)  \\
          &= (p - 1 / 2)(-3p^2 / 4 + 2p - (q - p / 2)^2).
      \end{align*}

      Since $p > 1 > 1 / 2$ and $p / 2 < 1 < q$, the supremum with $q \in (1, 2]$
      is reached when $q = 1$:
      \begin{align*}
        \Rev \leq (p - 1 / 2)(-3p^2 / 4 + 2p - (1 - p / 2)^2)
          = (p^3 - 5p^2 + 7p - 2) / 2.
      \end{align*}
      According to the first order condition, the local maximum and local minimum
      of the right-hand-side is reached when $p = 1$ and $p = 7 / 3$, respectively.
      In other words, the supremum with $p \in (1, 2]$ is reached when $p = 1$:
      \begin{align*}
        \Rev < (1^3 - 5 \cdot 1^2 + 7 \cdot 1 - 2) / 2 = 1 / 2 < 59 / 108.
      \end{align*}
    \end{proof}

\section{Comparison of Accuracy}\label{sec:compacc}

  \begin{figure}[bh]
    \phantom{1} \hfill %
    \subfigure[Mechanism via LP]
      {\includegraphics[width=0.37\textwidth]{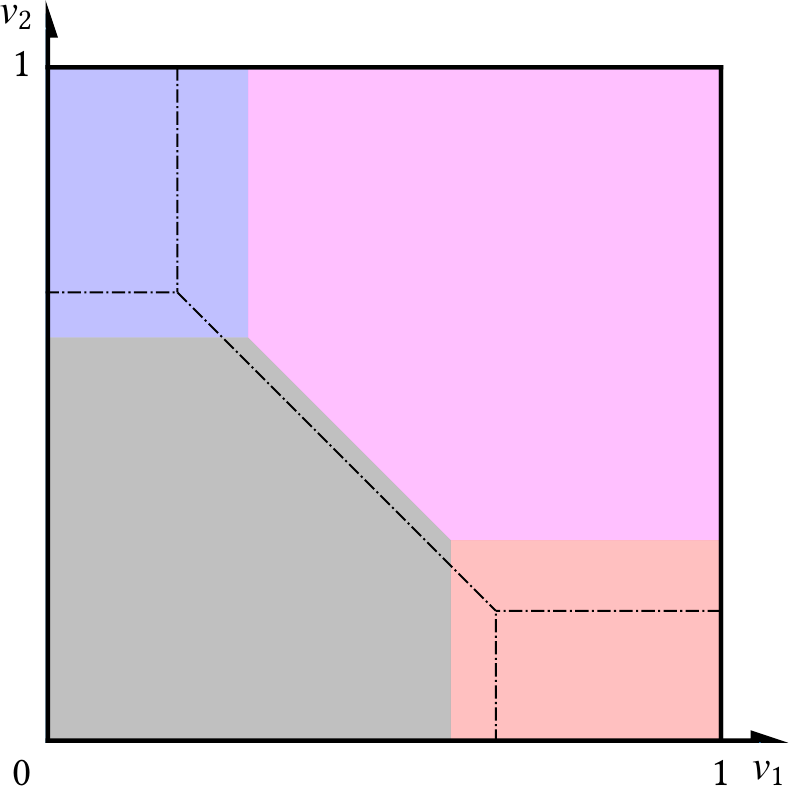}
      \label{sfig:compacc10lp}} %
    \hfill %
    \subfigure[Mechanism via our method]
      {\includegraphics[width=0.37\textwidth]{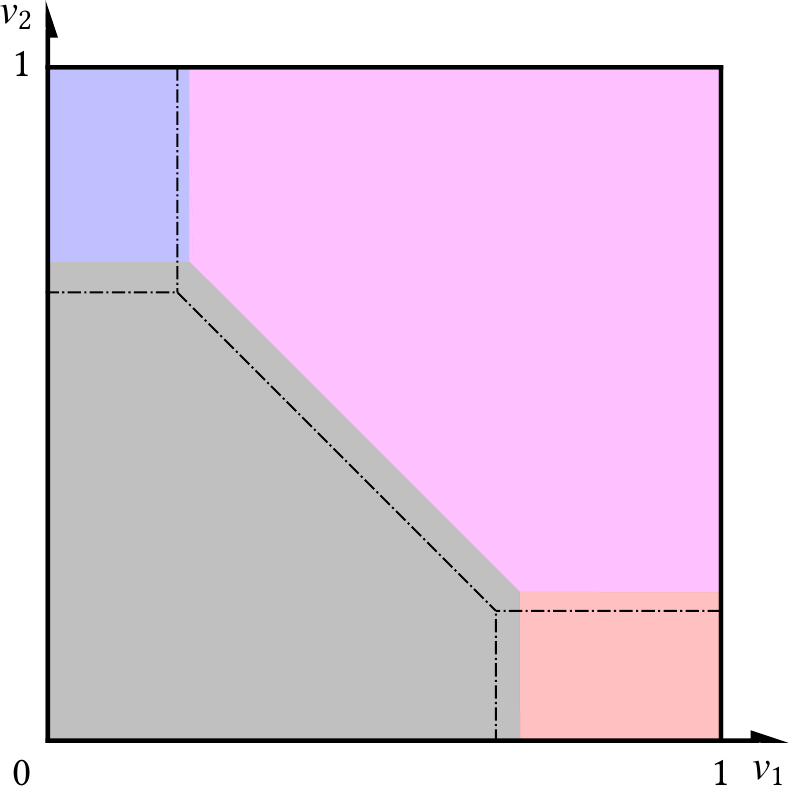}
      \label{sfig:compacc10nn}} %
    \hfill \phantom{1}
    \caption{Uniform $[0, 1]^2$ with discretization $N = 10$.}
    \label{fig:compacc10}
  \end{figure}

  \begin{figure}[bh]
    \phantom{1} \hfill %
    \subfigure[Mechanism via LP]
      {\includegraphics[width=0.37\textwidth]{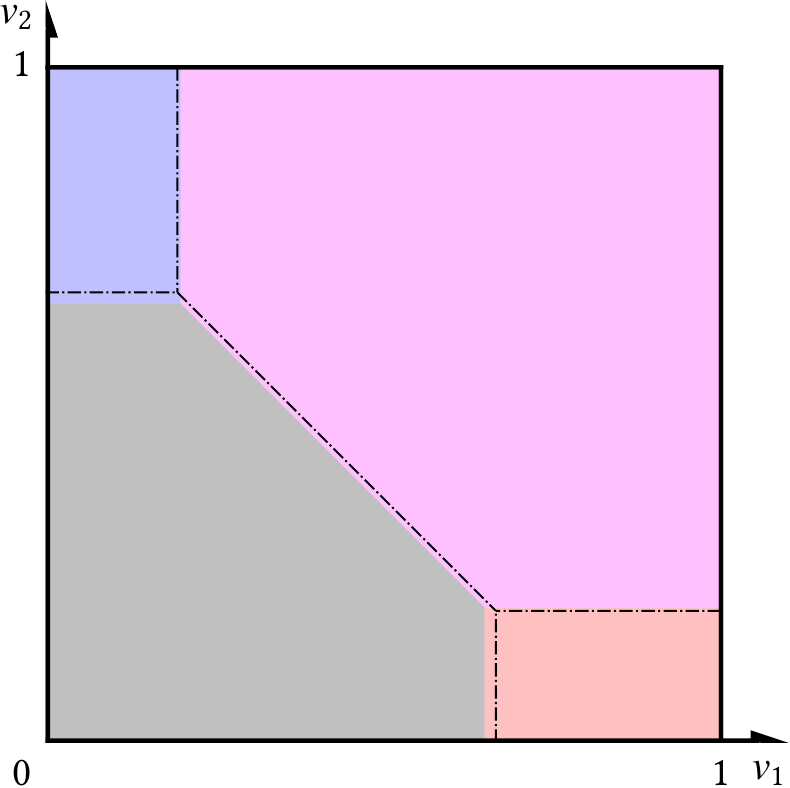}
      \label{sfig:compacc20lp}} %
    \hfill %
    \subfigure[Mechanism via our method]
      {\includegraphics[width=0.37\textwidth]{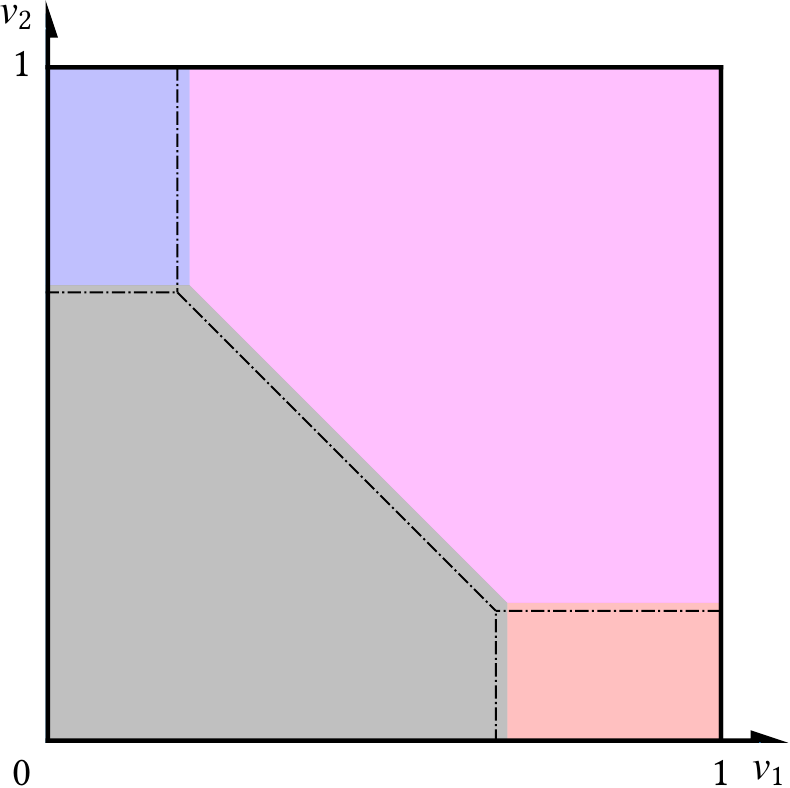}
      \label{sfig:compacc=20nn}} %
    \hfill \phantom{1}
    \caption{Uniform $[0, 1]^2$ with discretization $N = 20$.}
    \label{fig:compacc20}
  \end{figure}

  \begin{figure}[bh]
    \phantom{1} \hfill %
    \subfigure[Mechanism via LP]
      {\includegraphics[width=0.37\textwidth]{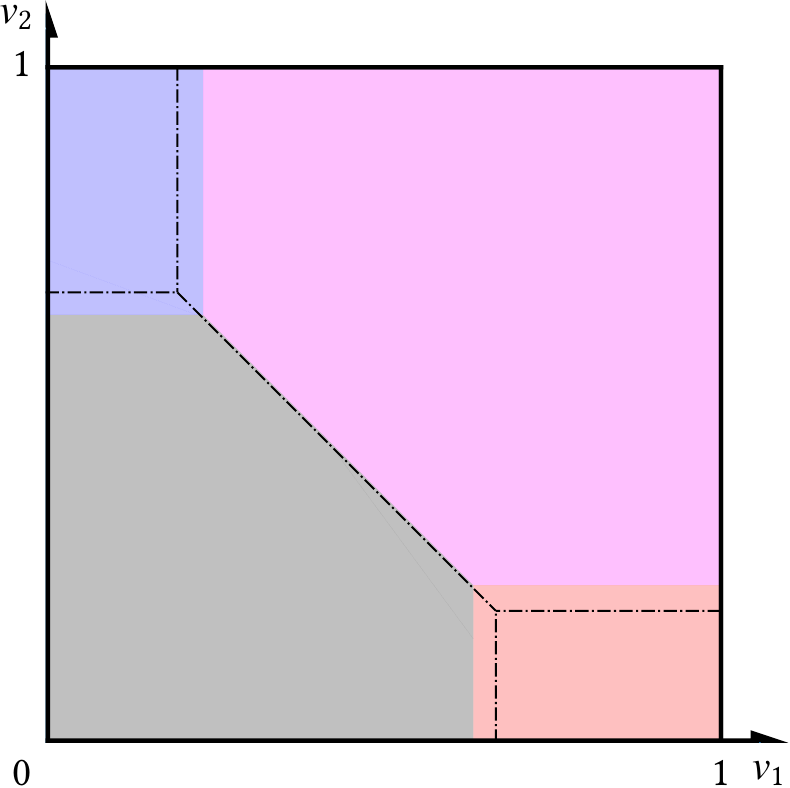}
      \label{sfig:compacc30lp}} %
    \hfill %
    \subfigure[Mechanism via our method]
      {\includegraphics[width=0.37\textwidth]{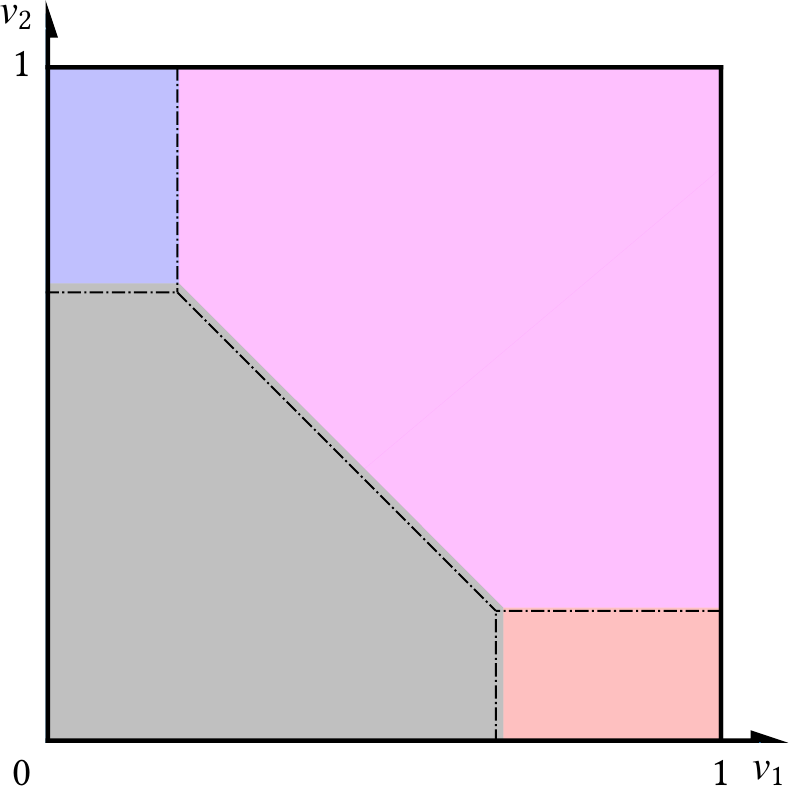}
      \label{sfig:compacc30nn}} %
    \hfill \phantom{1}
    \caption{Uniform $[0, 1]^2$ with discretization $N = 30$.}
    \label{fig:compacc30}
  \end{figure}



\end{document}